\pdfoutput=1
\documentclass[journal,10pt]{IEEEtran}

\usepackage[T1]{fontenc}
\usepackage[utf8]{inputenc}
\usepackage{lmodern}
\usepackage{microtype}
\usepackage{amsmath,amssymb,amsfonts,amsthm,mathtools,bm}
\usepackage{graphicx}
\usepackage{booktabs,array,tabularx,multirow}
\usepackage{enumitem}
\usepackage{xurl}
\usepackage{cite}
\usepackage{stfloats}
\usepackage{placeins}
\usepackage[colorlinks=true,linkcolor=blue,citecolor=blue,urlcolor=blue,hypertexnames=false]{hyperref}

\numberwithin{equation}{section}



\newtheorem{definition}{Definition}[section]
\newtheorem{construction}[definition]{Construction}
\newtheorem{theorem}[definition]{Theorem}

\newtheorem{proposition}[definition]{Proposition}
\newtheorem{assumption}[definition]{Assumption}
\theoremstyle{remark}

\newcolumntype{Y}{>{\raggedright\arraybackslash}X}
\setlist[itemize]{leftmargin=1.0em,itemsep=0.02em,topsep=0.04em}
\setlist[enumerate]{leftmargin=1.15em,itemsep=0.02em,topsep=0.04em}
\allowdisplaybreaks
\newcommand{\dd}{\,\mathrm d}

\newcommand{\transpose}{^{\mathsf T}}
\newcommand{\norm}[1]{\left\lVert #1\right\rVert}
\newcommand{\abs}[1]{\left\lvert #1\right\rvert}

\newcommand{\sgn}{\operatorname{sgn}}

\newcommand{\calA}{\mathcal A}
\newcommand{\calB}{\mathcal B}
\newcommand{\calC}{\mathcal C}
\newcommand{\calD}{\mathcal D}

\newcommand{\calF}{\mathcal F}
\newcommand{\calG}{\mathcal G}
\newcommand{\calK}{\mathcal K}
\newcommand{\calL}{\mathcal L}

\newcommand{\calN}{\mathcal N}
\newcommand{\calO}{\mathcal O}

\newcommand{\calR}{\mathcal R}
\newcommand{\calS}{\mathcal S}
\newcommand{\calT}{\mathcal T}

\newcommand{\Etot}{\mathcal E_{\rm tot}}
\newcommand{\Eint}{\mathcal E_{\rm int}}
\newcommand{\Eperp}{\mathcal E_{\perp}}
\newcommand{\Ecar}{\mathcal E_{\rm car}}
\newcommand{\Et}{E_{\rm tot}}
\newcommand{\Ep}{E_{\perp}}
\newcommand{\Mperp}{M_{\perp}}
\newcommand{\Meff}{M_{\rm eff}}
\newcommand{\Ppoe}{P_{\rm POE}}
\newcommand{\Psipoe}{\Psi_{\rm POE}}
\newcommand{\POE}{\mathrm{POE}}

\newcommand{\Fex}{F_{\rm exch}}
\newcommand{\qc}{q_{\rm c}}
\newcommand{\qy}{q_y}
\newcommand{\qck}{q_{{\rm c},k}}
\newcommand{\peta}{p_{\eta}}
\newcommand{\psperp}{p_{\sigma}^{\perp}}
\newcommand{\Akin}{A_{\rm kin}}
\newcommand{\Adyn}{A_{\rm dyn}}
\newcommand{\Akd}{A_{\rm kd}}
\newcommand{\Nadm}{N_{\rm adm}}
\newcommand{\phys}{\theta_{\rm phys}}
\newcommand{\vbar}{\bar v}
\newcommand{\neff}{n_{\rm e}}
\newcommand{\IP}{\mathrm{IP}}
\newcommand{\AP}{\mathrm{AP}}
\newcommand{\Jperp}{J_{\perp}}
\newcommand{\JL}{J_L}
\newcommand{\Rex}{\mathcal R_{\rm ex}}
\newcommand{\Sigmaori}{\Sigma_s}
\newcommand{\thetaStar}{\theta_*}
\newcommand{\Anonv}{A_{\rm nonv}^{\rm ptp}}
\newcommand{\Rrhs}{R_{\rm rhs}^{\rm int}}
\newcommand{\Rpoe}{R_{\rm POE}^{\rm fin}}

\newcommand{\Kell}{\mathrm K}
\newcommand{\Eell}{\mathrm E}

\title{Natural Locomotion: Principle and Method}
\author{Mirado Mortel, Luc Jaulin, Lionel Lapierre, and Simon Rohou%
\thanks{Preprint version.}}
\date{}

\begin{document}
\maketitle
\thispagestyle{plain}
\pagestyle{plain}

\begin{abstract}
Robotic locomotion can become efficient when mechanisms exploit passive dynamics, compliance, and resonance rather than track prescribed trajectories.  This paper formulates natural locomotion as a general exchange principle for systems whose net motion is mediated by environmental constraints or interactions.  A motion is natural when an internal oscillator returns periodically, the body pose drifts, and the mean Propulsion--Oscillator Exchange power (POE power) vanishes over one cycle.  The selected family is a Natural Locomotion Manifold (NLM).  We develop the conservative mathematical realization of this principle for continuous ideal environmental constraints: the constraints do no external work, total mechanical energy is conserved, and zero mean POE power is an internal exchange with the environment-mediated propulsive channel, not external energy input.

The method is a closed/open construction.  The propulsive channel is first closed to reveal an effective internal oscillator, organized by scalar action-angle structure in one effective degree of freedom or by nonlinear modal sectors in several degrees of freedom.  The channel is then reopened, pose is reconstructed, and accepted cycles must preserve internal recurrence and zero mean POE power.

We demonstrate the principle on two continuous ideal nonholonomic no-slip systems: a Chaplygin-sleigh / pendulum-driven-car and a three-body extension.  In the scalar case, POE closure is proved equivalent to the missing internal return condition, giving a theorem-backed computation of the NLM family.  In the multi-degree case, POE closure remains necessary but must be completed by modal identity, internal return, dynamics consistency, same fixed passive architecture, and nonzero displacement.  Natural locomotion therefore becomes a design question: a passive architecture may support no, one, or several certified NLM families.  The same principle suggests future realizations for discontinuous contacts in walking and continuous non-ideal environmental interactions in swimming or flying.
\end{abstract}

\begin{IEEEkeywords}
Natural locomotion, Natural Locomotion Manifold, nonholonomic locomotion, nonlinear normal modes, action-angle methods, Propulsion--Oscillator Exchange, Chaplygin sleigh, modal certification.
\end{IEEEkeywords}

\section{Introduction and Positioning}
\label{sec:intro}

\IEEEPARstart{B}{io-inspired} locomotion promises efficiency not by overpowering mechanics, but by exploiting them.  Animals and well-designed compliant robots benefit from resonance, elasticity, and passive dynamics because these effects allow motion to emerge with the system rather than against it.  This broad intuition underlies passive-dynamic walking, compliant locomotion, and resonance-based design \cite{McGeer1990,Collins2005,Kashiri2018,Schonebaum2021}.  It also motivates templates and anchors in biomechanics and robotics, where simplified mechanical structures expose the organization of efficient motion \cite{Blickhan1989,FullKoditschek1999,Holmes2006}.  The unresolved point is not whether passive mechanics can help.  It is what physically selects a natural locomotor motion for a given mechanism.

The classical modal answer is incomplete for locomotion.  Linear normal modes describe small oscillations near equilibrium, and nonlinear normal modes extend modal structure to finite amplitudes \cite{Rosenberg1962,ShawPierre1993,Kerschen2009PartI,Peeters2009PartII,AlbuSchaeffer2020,HallerPonsioen2016}.  In robotics, nonlinear modal and natural-motion-manifold control methods can excite, track, or stabilize efficient cyclic families once those families are known \cite{DellaSantina2021ISER,DellaSantina2021CDC,Bjelonic2022,Calzolari2023,Calzolari2026NMM}.  The question addressed here is prior to excitation or stabilization: which family should count as natural for the mechanism itself?  A locomotor cannot close its complete state while moving: the internal variables may recur, but the group pose must drift.  Geometric locomotion supplies this internal-loop/pose-drift split \cite{KellyMurray1995,OstrowskiBurdick1998,BlochKrishnaprasadMarsdenMurray1996,Bloch2015}; continuation methods then help trace branches and provide independent comparisons \cite{Raff2022,KuznetsovContinuation,DoedelOldeman2012,Rajaomarosata2025IFAC}.  Yet neither a reconstructing shape loop nor a continued branch is by itself a mechanical selection rule.

This paper formulates that rule at the level of environment-mediated locomotion and develops its conservative continuous-constraint realization.  The environmental mediator may be a contact, a fluid interaction, or, in the mathematical setting treated here, an ideal velocity constraint.  In the present realization the environment is not an actuator: the ideal constraint restricts admissible velocities, does no external work, and makes the propulsive coupling channel through which internal oscillation reconstructs into body displacement.  POE power is the bookkeeping variable for this constraint-mediated exchange; it tests whether the opened propulsive channel returns, over one internal cycle, the energy it borrows from the internal oscillator.  Figure~\ref{fig:principle} summarizes why natural locomotion replaces full-state recurrence by internal recurrence with pose drift.

The two demonstrations are continuous ideal nonholonomic no-slip systems.  The two-segment (2SEG) model is a Chaplygin-sleigh or pendulum-driven-car system: a knife-edge carrier carries one passive internal rotor with relative yaw angle.  The three-segment (3SEG) model is a three-body extension with two serial passive yaw joints.  The 2SEG has one effective internal degree of freedom after the propulsive channel is closed, so one scalar exchange condition can become a theorem-supported selector.  The 3SEG has two effective internal directions, so it requires closed-channel modal sectors before the opened-channel POE test can be interpreted.  Discontinuous contacts in walking and continuous non-ideal interactions in swimming or flying require different balance laws; they are future mathematical realizations of the same organizing question, not claims proved in this paper.

\paragraph*{Contributions}
This paper makes four contributions.
\begin{enumerate}
\item It formulates cyclic POE closure as a natural-locomotion principle for environment-mediated locomotion and develops the conservative continuous-ideal-constraint realization used in this paper.
\item It proves and implements the one-effective-internal-degree scalar case.  For the 2SEG under regularity assumptions, POE closure is equivalent to the missing scalar internal return, which yields a direct finite-speed computation of the NLM family.
\item It gives the closed-channel then opened-channel method for systems with several effective internal degrees of freedom.  In that setting POE remains necessary, but modal identity, state return, dynamics consistency, fixed physical parameters, and nonzero displacement must also be certified.
\item It frames a design question for passive locomotor architectures: for a fixed mechanism, are there zero, one, or multiple natural-locomotion families, and which mechanical features favor or suppress each mode?
\end{enumerate}

The paper proceeds in the construction order.  Section~\ref{sec:principle} states natural locomotion as internal recurrence, pose drift, and exchange closure.  Section~\ref{sec:action_and_models} relates the scalar action picture to the modal extension.  Section~\ref{sec:internal_mechanics} derives the closed/open split and the 2SEG mechanics; Sections~\ref{sec:2seg_theorem}--\ref{sec:2seg_results} give the exact scalar theorem, numerical realization, and validation.  Sections~\ref{sec:3seg_model_primer}--\ref{sec:3seg_results} give the 3SEG modal realization and same-physical result.  The discussion returns to the broader robotics question of architecture-supported natural modes and to how the conservative no-slip setting can serve as a reference point for later hybrid-contact, fluid-interaction, dissipative, forced, walking, swimming, or flying extensions.

\begin{figure*}[!t]
\centering
\includegraphics[width=0.98\textwidth]{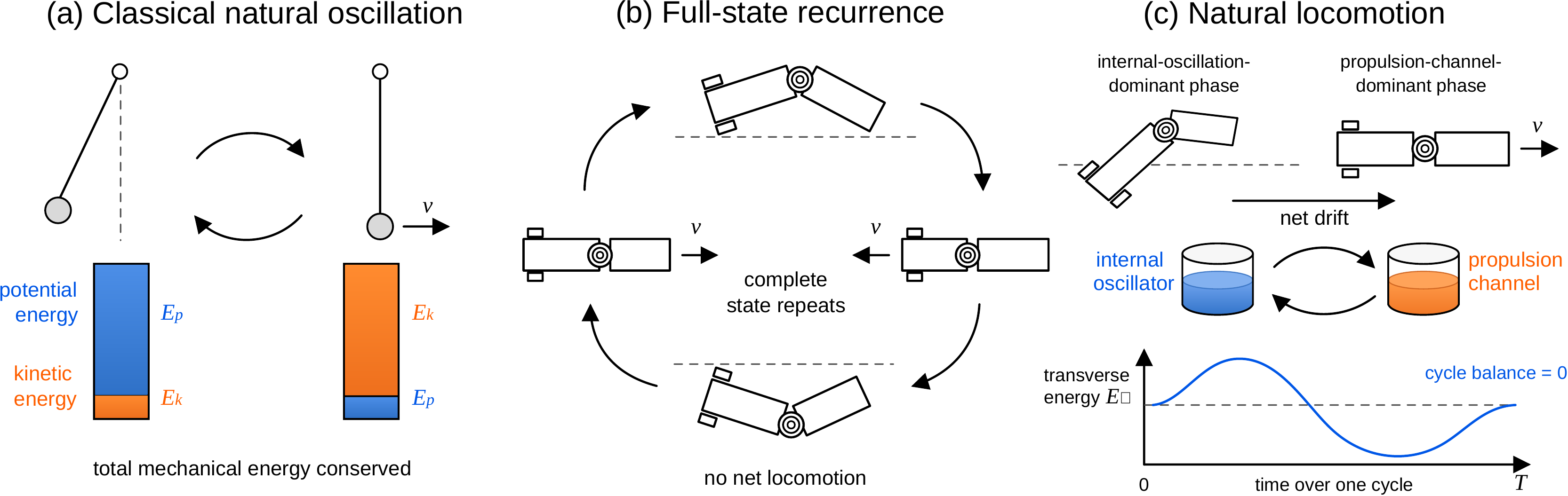}
\caption{Natural oscillation, full-state recurrence, and natural locomotion.  A classical oscillator closes its full state and redistributes kinetic and potential storage.  A locomotor must let pose drift, so recurrence is internal.  Natural locomotion additionally requires zero net exchange between the internal oscillator and the propulsive channel over one internal cycle.}
\label{fig:principle}
\end{figure*}

\section{Natural Oscillation and Natural Locomotion}
\label{sec:principle}

A natural oscillation is not a trajectory imposed on a mechanical system.  It is a recurrent organization that the mechanical system can sustain because its internal storage variables return.  The simple pendulum already separates this idea from a fixed-frequency picture.  Away from the small-amplitude limit, the period, waveform, and turning points change with energy.  The persistent object is a family of oscillations organized by storage exchange.  In one degree of freedom, a regular energy value determines turning points, branch velocity, quadrature, and action.  In several degrees of freedom, the same total energy generally contains many phase relations and modal partitions, so modal information must be added.

For an ordinary conservative oscillator,
\begin{equation}
  \Etot(q,\dot q)=T(q,\dot q)+V(q)=\Et .
\label{eq:min_etot}
\end{equation}
In a single degree of freedom a regular value of \(\Et\) describes a one-dimensional loop.  In several degrees of freedom it describes a larger energy surface and must be paired with a modal family.  The dimensional distinction is
\begin{equation}
\begin{array}{rcl}
  1\hbox{-DOF} &:& \Et \mapsto \hbox{one regular oscillatory loop},\\[0.2em]
  N\hbox{-DOF} &:& (\Et,k) \mapsto \hbox{one modal oscillatory family}.
\end{array}
\label{eq:oscillator_dim_balance}
\end{equation}

For a locomotor with complete state \(X=(g,z)\), where \(g\) is the group pose and \(z\) denotes internal variables, nonzero displacement prevents complete recurrence:
\begin{equation}
  X(t+T)\neq X(t).
\label{eq:no_full_recurrence}
\end{equation}
The recurrence compatible with locomotion is instead
\begin{equation}
  z(t+T)=z(t),\qquad
  g(t+T)=g(t)\Delta g,\qquad
  \Delta g\neq e .
\label{eq:internal_pose_drift}
\end{equation}
Equation \eqref{eq:internal_pose_drift} is the geometric skeleton.  It distinguishes internal closure from pose drift, but it does not decide whether the internal loop is mechanically natural.

The POE row supplies that decision.  Let \(\Ppoe(t)\) be the instantaneous power exchanged between the internal oscillator and the propulsive channel.  In the conservative ideal-constraint realization, this is not the mechanical power of an external force.  Ideal constraints do no external work; their role is to restrict admissible velocities, and that restriction mediates an internal redistribution between oscillatory storage and carrier motion.  The naturality condition does not remove instantaneous exchange; it removes net exchange over one completed internal cycle:
\begin{equation}
  \Psipoe
  =\frac{1}{T}\int_0^T \Ppoe(t)\dd t=0 .
\label{eq:poe_balance_min}
\end{equation}
\enlargethispage{4pt}
The normalization by \(T\) is immaterial for the zero but useful when comparing cycles.  A natural-locomotion family can then be written
\begin{equation}
  \calN_k=\{\Gamma_{k,\vbar}:\vbar\in I_k,\;
  z(T)=z(0),\;\Delta g\neq e,\;\Psipoe=0\},
\label{eq:nlm_family_min}
\end{equation}
with \(k\) denoting either the scalar sector or a nonlinear modal sector.

The word ``manifold'' is used in this reduced mechanical sense.  On a regular branch and after quotienting the global pose, the NLM is a two-dimensional invariant object in the reduced state: one coordinate is phase along a closed internal cycle and the second is the family parameter, such as energy, mean speed, action, or amplitude.  In the unreduced locomotor state, the same object lifts to a relative-periodic family because each internal period is accompanied by a group displacement.  Figure~\ref{fig:3seg_surface} gives a concrete preview of that meaning for the three-segment case: it shows two such accepted reduced families at one fixed passive architecture, with the six top plots corresponding to anti-phase samples and the six bottom plots to in-phase samples.

\begin{figure*}[!t]
\centering
\includegraphics[width=0.985\textwidth]{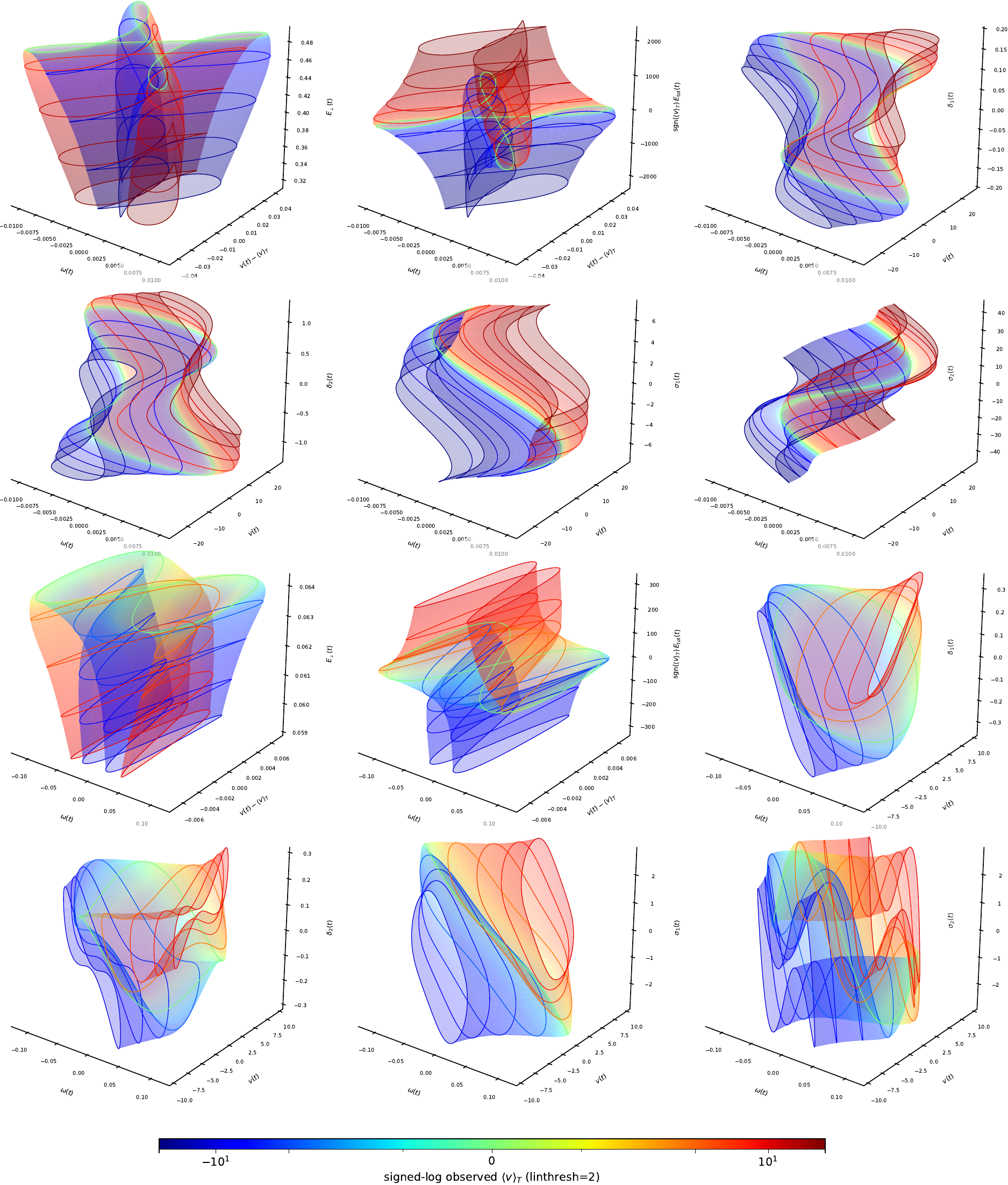}
\caption{Omega--speed--transverse-energy surface visualization for sampled accepted three-segment cycles at the same fixed passive architecture.  The six top plots are anti-phase (AP) samples and the six bottom plots are in-phase (IP) samples.  The surface view is used here as a visual example of Natural Locomotion Manifold structure: a regular accepted family is represented, in reduced variables, by phase along the cycle and a family parameter.  The same figure is interpreted quantitatively in the 3SEG results.}
\label{fig:3seg_surface}
\end{figure*}

The same statement has a useful mechanical reading at increasing resolution.  The word ``closed'' means that the admissible pseudo-momentum coordinate driving the carrier is set to zero, so the internal oscillator can be read without propulsive drift.  The word ``opened'' means that this channel is restored and the resulting cycle is tested for internal return, nonzero pose drift, and zero net exchange.  Table~\ref{tab:principle_reading} separates the physical sentence from the rows that are later implemented numerically.  It also records where the scalar and modal cases diverge.

\begin{table}[!t]
\caption{Same natural-locomotion principle at increasing mechanical resolution.}
\label{tab:principle_reading}
\centering
\footnotesize
\setlength{\tabcolsep}{3pt}
\renewcommand{\arraystretch}{1.00}
\begin{tabularx}{\columnwidth}{@{}p{0.39\columnwidth}X@{}}
\toprule
Physical statement & Mechanical reading \\
\midrule
Natural oscillator & Transverse support after closing the carrier pseudo-momentum channel. \\
Propulsion channel closed & \(\qc=0\), zero admissible pseudo-momentum leaf. \\
Propulsion channel opened & \(\qc(t)\neq0\), finite-speed internal lift. \\
Zero net POE & \(\Eperp(T)-\Eperp(0)=0\). \\
One-effective-degree closure & POE closes the remaining scalar section coordinate. \\
Multi-effective-degree closure & POE plus modal, state, branch, and same-physical certificates. \\
\bottomrule
\end{tabularx}
\end{table}

The construction may be summarized as two maps.  The carrier-closed map constructs internal supports,
\begin{equation}
  (\theta_{\rm phys},k,a)
  \longmapsto
  \gamma_{\perp,k}(a;\theta_{\rm phys}),
\label{eq:closed_support_map}
\end{equation}
where \(a\) is an amplitude, action, or continuation coordinate.  The opened-channel map tests survival under propulsion,
\begin{equation}
  \gamma_{\perp,k}
  \xrightarrow{\calL_{\rm int}}
  z_{\perp,k}
  \xrightarrow{\calL_{\rm open}}
  \Gamma_{k,\vbar}
  \xrightarrow{\calR_{\rm POE}}
  \hbox{return residual}.
\label{eq:opened_test_map}
\end{equation}
For \(\neff=1\), the return residual can become exact by a scalar storage row after the carrier and speed rows are closed.  For \(\neff>1\), the residual must remember the nonlinear modal sector that generated the support.

This principle fixes the computational order, illustrated in Fig.~\ref{fig:closedopen}.  The propulsive channel is first closed.  This exposes the internal oscillator and its natural supports.  The propulsive channel is then reopened.  The opened lift is accepted only if the internal state returns, the exchange balance closes, and the pose accumulates nonzero displacement.  The closed/open operation prevents a common confusion: a closed shape loop can be geometrically useful without being natural, and a natural internal oscillator can fail to survive propulsion if the opened channel leaves a net transverse-storage imbalance.

\begin{figure*}[!t]
\centering
\includegraphics[width=0.98\textwidth]{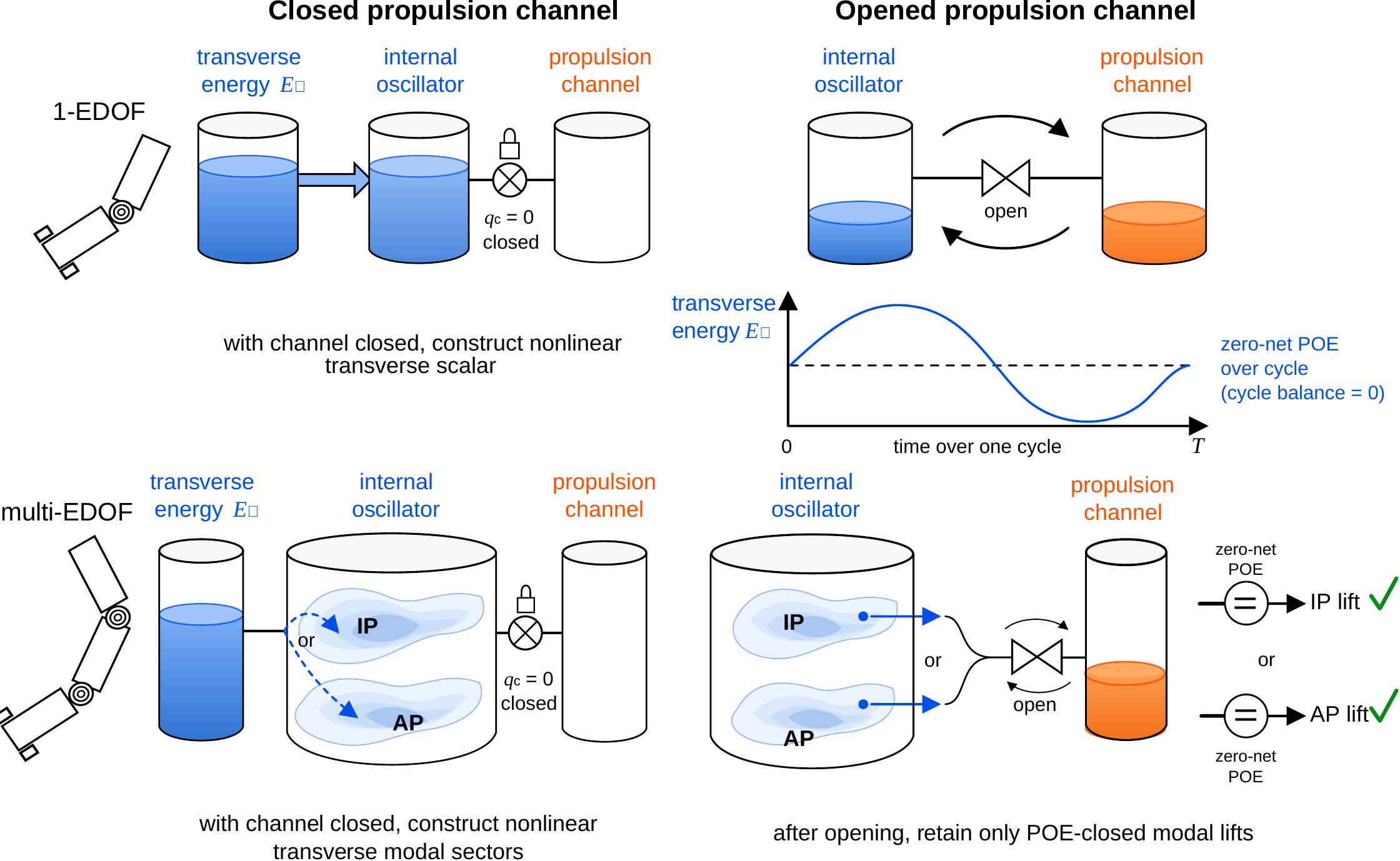}
\caption{Closed/open construction.  In one effective internal degree of freedom, the closed channel gives a scalar transverse oscillator; after opening, POE closes the remaining scalar storage coordinate.  In multiple effective internal degrees of freedom, the closed channel first supplies nonlinear transverse modal sectors such as in-phase (IP) and anti-phase (AP); after opening, POE is necessary but must be paired with modal and state certificates.}
\label{fig:closedopen}
\end{figure*}

\section{Action, Quadrature, and Modal Organization}
\label{sec:action_and_models}

The scalar template is the conservative one-degree-of-freedom oscillator
\begin{equation}
  \Etot(q,\dot q)=\frac12M(q)\dot q^2+V(q)=\Et .
\label{eq:onecoord_energy}
\end{equation}
A regular libration at energy \(\Et\) has turning points \(q_-(\Et)\) and \(q_+(\Et)\), branch velocity
\begin{equation}
  \dot q_\pm(q;\Et)=\pm\sqrt{\frac{2(\Et-V(q))}{M(q)}} ,
\label{eq:qdot_branch}
\end{equation}
and action
\begin{equation}
  J(\Et)=\frac1\pi\int_{q_-}^{q_+}\sqrt{2M(q)(\Et-V(q))}\dd q .
\label{eq:action_scalar}
\end{equation}
The frequency follows from the action-energy relation, \(\Omega(J)=\dd \Et/\dd J\), when the regular action chart is valid \cite{Arnold1989,NayfehMook1995,SandersVerhulstMurdock2007}.  This scalar picture is the prototype for the 2SEG transverse oscillator.

The simple pendulum makes the same point in a familiar closed form.  With
\begin{equation}
  \Etot(\theta,p_\theta)=\frac{p_\theta^2}{2m\ell^2}+mg\ell(1-\cos\theta),
\label{eq:pendulum_energy}
\end{equation}
and \(k=\sin(\theta_{\max}/2)\) below the separatrix, the finite-amplitude period and action can be expressed through complete elliptic integrals,
\begin{subequations}
\label{eq:pendulum_action}
\begin{align}
  T(k)&=4\sqrt{\frac{\ell}{g}}\,\Kell(k),\\
  J(k)&=\frac{8m\ell^2\sqrt{g/\ell}}{\pi}
  \left[\Eell(k)-(1-k^2)\Kell(k)\right].
\end{align}
\end{subequations}
This elementary example is included to fix the meaning of ``natural.''  Even before locomotion, a natural family is not a single small-signal frequency; it is a storage-indexed family with a time law determined by the mechanics.  Figure~\ref{fig:action2seg} transfers this scalar action picture to the 2SEG transverse storage used below.

\begin{figure*}[!t]
\centering
\includegraphics[width=1.00\textwidth]{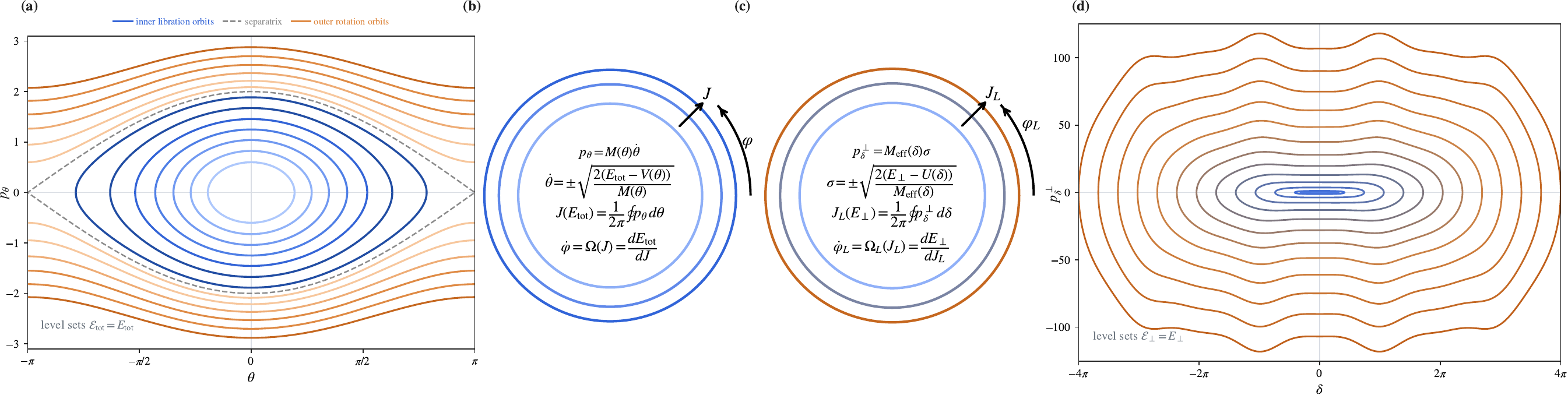}
\caption{Action organization and the scalar 2SEG analogue.  The one-degree-of-freedom action picture organizes a family of oscillations by energy and action.  The carrier-closed two-segment transverse dynamics has the analogous scalar storage \(\Eperp\), action-like coordinate \(J_L\), and finite-amplitude support before the propulsive channel is reopened.}
\label{fig:action2seg}
\end{figure*}

The two-segment system has one effective transverse coordinate.  Its carrier-closed transverse storage is
\begin{equation}
  \Eperp(\delta,\sigma)=\frac12\Meff(\delta)\sigma^2+U(\delta)=\Ep .
\label{eq:2seg_Eperp}
\end{equation}
Thus
\begin{equation}
  \sigma_\pm(\delta;\Ep)=\pm\left[\frac{2(\Ep-U(\delta))}{\Meff(\delta)}\right]^{1/2},
\label{eq:2seg_sigma_branch}
\end{equation}
with turning angles \(\Ep-U(\delta_\pm)=0\), period
\begin{equation}
  \frac{T_L(\Ep)}{2}=\int_{\delta_-}^{\delta_+}
  \sqrt{\frac{\Meff(\delta)}{2(\Ep-U(\delta))}}\dd\delta,
\label{eq:2seg_period}
\end{equation}
and action-like coordinate
\begin{equation}
  \JL(\Ep)=\frac1\pi\int_{\delta_-}^{\delta_+}
  \sqrt{2\Meff(\delta)(\Ep-U(\delta))}\dd\delta .
\label{eq:2seg_action}
\end{equation}
This closed support exposes the scalar oscillator.  The opened theorem in Section~\ref{sec:2seg_theorem} concerns whether that oscillator survives the yaw-propulsion channel.

The three-segment system has two effective internal shape directions.  With \(y=(r,\sigma)\), \(\eta=(v,\omega)\), and block mass matrix
\begin{equation}
M(r)=
\begin{bmatrix}
M_{GG}(r)&M_{GS}(r)\\
M_{SG}(r)&M_{SS}(r)
\end{bmatrix},
\label{eq:mass_split_intro}
\end{equation}
the carrier-closed transverse layer is defined by the Schur complement
\begin{subequations}
\label{eq:schur_layer_intro}
\begin{align}
A(r)&=M_{GG}^{-1}(r)M_{GS}(r),\qquad
\eta_\perp=-A(r)\sigma,\\
\Mperp(r)&=M_{SS}-M_{SG}M_{GG}^{-1}M_{GS},\\
\Eperp(r,\sigma)&=\frac12\sigma\transpose\Mperp(r)\sigma+U(r).
\end{align}
\end{subequations}
The closed transverse field preserves \(\Eperp\), but \(\Eperp=\Ep\) does not select one loop.  Modal sectors must be constructed.  Around the transverse origin, parent IP/AP directions are obtained from
\begin{equation}
  K_0u_j=\Omega_j^2M_0u_j,
  \qquad
  u_i\transpose M_0u_j=\delta_{ij},
  \qquad j\in\{\IP,\AP\}.
\label{eq:parent_modes_intro}
\end{equation}
Finite-amplitude transverse supports are then continued from these seeds before being lifted into the opened locomotor dynamics.

A finite-amplitude support is therefore not accepted solely because it preserves energy.  The modal sector supplies the missing organization,
\begin{equation}
\begin{gathered}
  \hbox{energy level} + \hbox{modal sector}\\
  \longrightarrow
  \hbox{transverse natural-oscillation family}.
\end{gathered}
\label{eq:modal_completion_template}
\end{equation}
The NLM construction applies this completion before pose reconstruction.  The modal object is the carrier-closed internal oscillator, not a periodic full-state locomotor orbit.  The shared nonholonomic convention for the scalar and modal examples is shown in Fig.~\ref{fig:kinematics}: both systems have a constrained carrier, but they differ in the number of internal transverse directions.

\begin{figure}[!t]
\centering
\includegraphics[width=1.00\columnwidth]{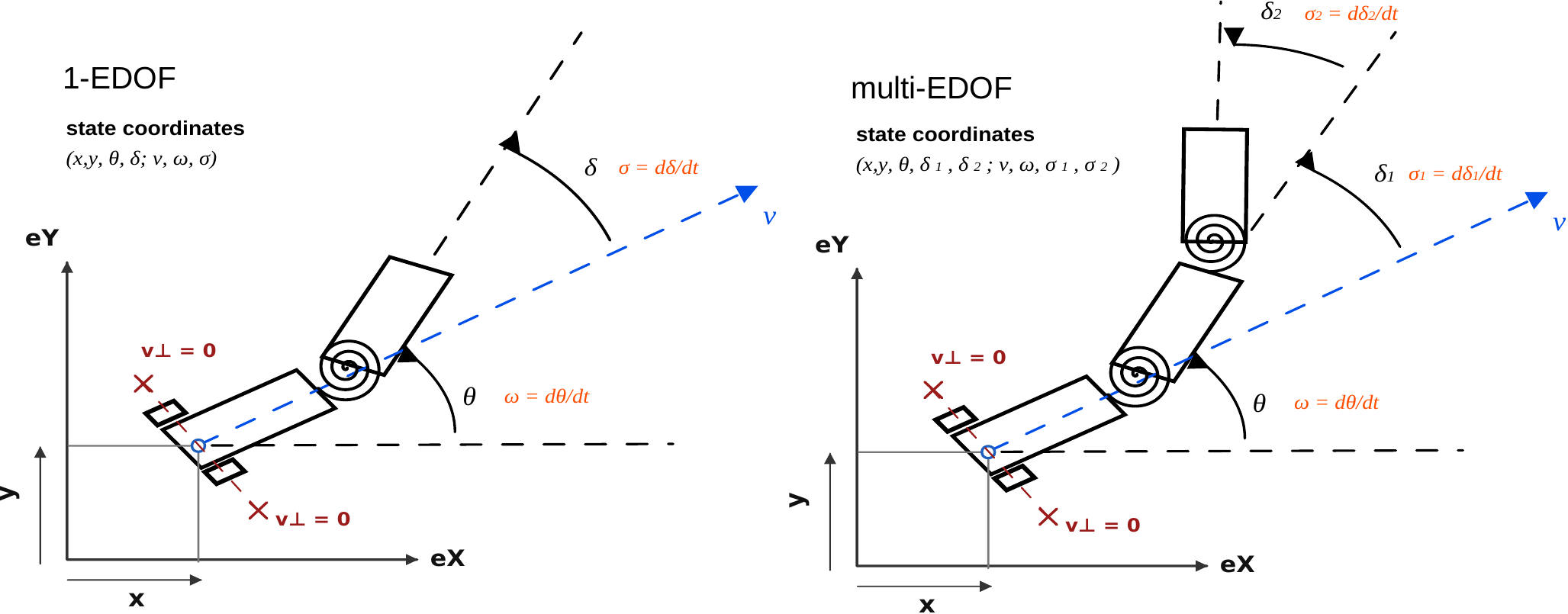}
\caption{Kinematic convention for the two-segment (2SEG) and three-segment (3SEG) nonholonomic models.  Both models share a carrier pose and a knife-edge/no-side-slip velocity constraint; they differ in the number of internal transverse directions available to the oscillator.}
\label{fig:kinematics}
\end{figure}

\section{Internal Mechanics and the Propulsive Channel}
\label{sec:internal_mechanics}

The principle becomes computable only after identifying which mechanical coordinate represents the propulsive channel and which storage remains when that channel is closed.  Mechanically, the propulsive channel is the admissible group-motion momentum that would move the carrier through the nonholonomic constraint.  In the ideal nonholonomic setting, constraint forces are eliminated by projection and do no external work, so total mechanical energy is conserved.  Conservation nevertheless permits internal exchange: the constraint restricts velocities, this restriction mediates the carrier coupling, and the opened coupling redistributes energy between the carrier channel and the transverse oscillator.  Closing the channel is not deleting the carrier from the model; it is setting this momentum to zero to expose the oscillator that the carrier motion would otherwise exchange energy with.  The configuration is split into a group pose \(g\in G\) and shape variables \(r\in S\).  After symmetry reduction and constraint compatibility, the complete state is written
\begin{equation}
  X=(g,r,\eta,\sigma),
  \qquad \sigma=\dot r,
\label{eq:complete_state_gen}
\end{equation}
where \(\eta\) denotes admissible group-velocity coordinates expressed internally.  The pose-reduced internal state is
\begin{equation}
  z=(r,\eta,\sigma),
\label{eq:internal_state_gen}
\end{equation}
and the transverse state is
\begin{equation}
  y=(r,\sigma).
\label{eq:transverse_state_gen}
\end{equation}
The word transverse is mechanical, not only notational: it refers to the oscillator obtained after the propulsive pseudo-momentum channel has been closed.

Two reconstructions must be kept distinct.  The internal lift maps a transverse motion into internal variables,
\begin{equation}
  y\mapsto z,
\label{eq:internal_lift_gen}
\end{equation}
whereas pose reconstruction maps an internal trajectory into the complete state,
\begin{equation}
  z\mapsto X,
  \qquad g^{-1}\dot g=\Xi(r,\eta,\sigma).
\label{eq:pose_reconstruction_gen}
\end{equation}
The propulsive channel lives inside the internal state, not in the absolute pose.  It is the admissible group-momentum channel through which the ideal-constraint coupling can exchange energy with the oscillator while the complete conservative system preserves total mechanical energy.

Let \(\xi\in\mathfrak g\) be the body velocity associated with \(g\).  The constrained Lagrange-d'Alembert equation on complete coordinates \(x=(g,r)\) is
\begin{equation}
  \frac{\dd}{\dd t}\frac{\partial L}{\partial \dot x}
  -\frac{\partial L}{\partial x}
  =\calA(x)^{\mathsf T}\lambda,
\label{eq:LdA_full_gen}
\end{equation}
with Pfaffian constraints \(\calA(x)\dot x=0\).  If \(N(x)\) spans \(\ker\calA(x)\), admissible velocities satisfy \(\dot x=N(x)u_a\), and multiplication by \(N(x)^{\mathsf T}\) removes the multipliers:
\begin{equation}
  N(x)^{\mathsf T}\left(\frac{\dd}{\dd t}\frac{\partial L}{\partial \dot x}
  -\frac{\partial L}{\partial x}\right)=0.
\label{eq:LdA_projected_gen}
\end{equation}
The Pfaffian constraints may be written
\begin{equation}
  \calA_G(r)\xi+\calA_S(r)\sigma=0.
\label{eq:pfaffian_gen}
\end{equation}
A choice of admissible basis gives
\begin{equation}
  \xi=-\Akin(r)\sigma+B(r)\eta,
\label{eq:admissible_split_gen}
\end{equation}
where the first term is the kinematic connection induced by the constraints and the second spans remaining admissible group-velocity directions.  In matrix form,
\begin{equation}
  \begin{bmatrix}\xi\\ \sigma\end{bmatrix}
  =\Nadm(r)
  \begin{bmatrix}\eta\\ \sigma\end{bmatrix},
  \qquad
  \Nadm(r)=
  \begin{bmatrix}B(r)&-\Akin(r)\\0&I\end{bmatrix}.
\label{eq:Nadm_gen}
\end{equation}
Substituting \eqref{eq:Nadm_gen} into the kinetic energy pulls the full metric onto admissible velocities,
\begin{equation}
  M_{\rm int}(r)=\Nadm(r)^{\mathsf T}M_{\rm full}(r)\Nadm(r)
  =\begin{bmatrix}M_{GG}&M_{GS}\\ M_{SG}&M_{SS}\end{bmatrix}.
\label{eq:Mint_gen}
\end{equation}
The internal conservative energy is
\begin{align}
  \Eint
  &=\frac12
  \begin{bmatrix}\eta\\ \sigma\end{bmatrix}^{\mathsf T}
  M_{\rm int}(r)
  \begin{bmatrix}\eta\\ \sigma\end{bmatrix}+U(r) \notag\\
  &=\frac12\eta^{\mathsf T}M_{GG}\eta
  +\eta^{\mathsf T}M_{GS}\sigma
  +\frac12\sigma^{\mathsf T}M_{SS}\sigma+U(r).
\label{eq:Eint_expanded_gen}
\end{align}
The admissible group pseudo-momentum is the momentum associated with \(\eta\),
\begin{equation}
  \peta=\frac{\partial \Eint}{\partial \eta}
  =M_{GG}\eta+M_{GS}\sigma.
\label{eq:peta_gen}
\end{equation}
Assuming \(M_{GG}\) is nonsingular on the considered domain, define
\begin{equation}
\begin{aligned}
  \Adyn(r)&=M_{GG}^{-1}M_{GS},\\
  \qc&=M_{GG}^{-1}\peta=\eta+\Adyn(r)\sigma .
\end{aligned}
\label{eq:Adyn_qc_gen}
\end{equation}
Thus \(\eta=\qc-\Adyn(r)\sigma\).  Completing the square gives
\begin{equation}
\begin{aligned}
  \Eint
  ={}&\frac12\qc^{\mathsf T}M_{GG}\qc
  +\frac12\sigma^{\mathsf T}\Mperp\sigma+U(r),\\
  \Mperp
  ={}&M_{SS}-M_{SG}M_{GG}^{-1}M_{GS}.
\end{aligned}
\label{eq:complete_square_gen}
\end{equation}
Therefore
\begin{equation}
  \Eint=\Ecar+\Eperp,
  \qquad
  \Ecar=\frac12\qc^{\mathsf T}M_{GG}\qc,
\label{eq:energy_split_gen}
\end{equation}
with transverse storage
\begin{equation}
  \Eperp(r,\sigma)=\frac12\sigma^{\mathsf T}\Mperp(r)\sigma+U(r).
\label{eq:Eperp_general_gen}
\end{equation}
The closed propulsive channel is the zero-pseudo-momentum leaf
\begin{equation}
  \qc=0
  \quad\Longleftrightarrow\quad
  \peta=0
  \quad\Longleftrightarrow\quad
  \eta_\perp=-\Adyn(r)\sigma .
\label{eq:closed_leaf_gen}
\end{equation}
Combining \eqref{eq:closed_leaf_gen} with \eqref{eq:admissible_split_gen} gives the carrier-closed body velocity
\begin{equation}
  \xi_\perp=-\bigl(\Akin(r)+B(r)\Adyn(r)\bigr)\sigma
  =-\Akd(r)\sigma,
\label{eq:Akd_gen}
\end{equation}
where \(\Akd(r)=\Akin(r)+B(r)\Adyn(r)\) is the kinodynamic connection.  The transverse momentum and action associated with the closed oscillator are
\begin{equation}
  \psperp=\Mperp(r)\sigma,
  \qquad
  \Jperp=\frac{1}{2\pi}\oint_{\gamma_\perp}\psperp\cdot\dd r .
\label{eq:transverse_action_general_gen}
\end{equation}
The same admissible momentum split therefore defines both the propulsive channel \(\qc\) and the action of the closed internal oscillator.

\begin{proposition}[Closed-channel oscillator and opened-channel exchange]
\label{prop:closed_open_energy_gen}
Assume \(M_{GG}\) and \(\Mperp\) are positive definite on the considered region.  The zero-pseudo-momentum condition \(\qc=0\) makes \(\Eperp\) the conservative storage of the internal oscillator.  If the channel is opened by \(\qc(t)\neq0\), the instantaneous POE power is the defect of this storage along the opened internal dynamics:
\begin{equation}
  \Ppoe(t)=\frac{\dd}{\dd t}\Eperp(r(t),\sigma(t)).
\label{eq:PPOE_def_prop_gen}
\end{equation}
Consequently, a lifted cycle has zero net Propulsion--Oscillator Exchange if and only if
\begin{equation}
  \Eperp(r(T),\sigma(T))=\Eperp(r(0),\sigma(0)).
\label{eq:Eperp_return_prop_gen}
\end{equation}
\end{proposition}

\begin{proof}
The closed condition \(\qc=0\) removes \(\Ecar\) from \eqref{eq:energy_split_gen}; the remaining internal storage is \(\Eperp\).  When \(\qc\) is nonzero, the full internal dynamics no longer coincide with the Hamiltonian flow of \(\Eperp\), so the time derivative of \(\Eperp\) measures the power exchanged with the opened propulsive channel.  This POE power is an internal balance term, not an external-force power: in the ideal constrained model total mechanical energy is conserved, while \(\Eperp\) and \(\Ecar\) can exchange through the opened channel.  Integrating \eqref{eq:PPOE_def_prop_gen} over a period gives \eqref{eq:Eperp_return_prop_gen}.
\end{proof}

The split also explains why an NLM is not a full-state nonlinear normal mode (NNM).  A full-state NNM would impose
\begin{equation}
(g(t+T),r(t+T),\eta(t+T),\sigma(t+T))=(g(t),r(t),\eta(t),\sigma(t)),
\label{eq:full_state_closed_gen}
\end{equation}
which includes \(g(t+T)=g(t)\).  Natural locomotion instead uses
\begin{equation}
(r(t+T),\eta(t+T),\sigma(t+T))=(r(t),\eta(t),\sigma(t)),
\label{eq:internal_closed_gen}
\end{equation}
with a nontrivial group increment.  Internal recurrence is therefore posed on \(z\), while the closed oscillator is posed on \(y\) after the propulsive pseudo-momentum channel has been removed.

\subsection{2SEG Mechanical Model and Scalar Reduction}
\label{subsec:2seg_model_primer}

The 2SEG mechanism is the one-shape pendulum-driven car used for the scalar theorem and direct numerical realization.  It is a Chaplygin-sleigh carrier with one passive internal appendage.  Its configuration is
\begin{equation}
  Q=SE(2)\times S^1,\qquad q=(g,\delta),
\label{eq:2seg_config_model}
\end{equation}
where \(g\) is the planar carrier pose and \(\delta\) is the appendage yaw relative to the carrier.  In the body frame,
\begin{equation}
  \xi=g^{-1}\dot g=(v_x,v_y,\omega),\qquad \sigma=\dot\delta .
\label{eq:2seg_body_vel_model}
\end{equation}
Here \(v_x\) is forward speed, \(v_y\) is lateral speed, \(\omega\) is carrier yaw rate, and \(\sigma\) is internal shape rate.  The ideal knife-edge constraint is \(v_y=0\), so the admissible quasi-velocity is
\begin{equation}
  \nu=(v,\omega,\sigma)^{\mathsf T},\qquad v=v_x .
\label{eq:2seg_quasivel_model}
\end{equation}
The base has mass \(m_1\), planar yaw inertia \(I_1\), and center of mass at distance \(l_1\) from the knife-edge contact/pivot along the body axis.  The internal appendage has mass \(m_2\), planar yaw inertia \(I_2\), and center of mass at distance \(l_2\) from the pivot along the appendage.  The length \(l\) is the nondimensionalization scale, taken as \(l=l_1+l_2\) in the compact ratios below.  The potential \(U(\delta)\) is the nondimensional passive torsional potential; a dimensional source potential \(V\) is divided by a chosen energy scale to obtain \(U\).  After the constraint force is projected out, the conservative reduced model has the form
\begin{equation}
  \dot\delta=\sigma,\qquad
  M(\delta)\dot\nu+h(\delta,\nu)+[0,0,U'(\delta)]^{\mathsf T}=0,
\label{eq:2seg_reduced_model}
\end{equation}
with energy identity
\begin{equation}
  E(\delta,\nu)=\tfrac12\nu^{\mathsf T}M(\delta)\nu+U(\delta),
  \qquad \dot E=0 .
\label{eq:2seg_energy_model}
\end{equation}
With these physical meanings fixed, the nondimensional parameters used below are
\begin{equation}
\begin{gathered}
  \alpha=l_1/l,
  \qquad \beta=l_2/l,
  \qquad \mu=m_2/(m_1+m_2),\\
  j_i=I_i/((m_1+m_2)l^2).
\end{gathered}
\label{eq:2seg_nd_params_model}
\end{equation}
The reported scalar realization uses
\begin{equation}
  \alpha=\varepsilon,
  \quad \beta=1-\varepsilon,
  \quad \mu=1-\varepsilon,
  \quad j_1=\gamma\alpha^3,
  \quad j_2=\gamma\beta^3 .
\label{eq:2seg_eps_specialization_model}
\end{equation}
In this notation the nondimensional mass matrix contains the entries
\begin{equation}
\bar M(\delta)=
\begin{bmatrix}
1 & -\rho & -\rho\\
-\rho & B_{11} & B_{12}\\
-\rho & B_{12} & B_{22}
\end{bmatrix},
\qquad \rho(\delta)=\mu\beta\sin\delta,
\label{eq:2seg_mass_entries_model}
\end{equation}
with
\begin{equation}
\begin{aligned}
B_{11}(\delta)&=j_1+j_2+\alpha^2+\mu\beta^2
 +2\mu\alpha\beta\cos\delta,\\
B_{12}(\delta)&=j_2+\mu\beta^2+\mu\alpha\beta\cos\delta,\\
B_{22}&=j_2+\mu\beta^2 .
\end{aligned}
\label{eq:2seg_B_entries_model}
\end{equation}
The scalar inertia \(\Meff\), determinant \(\Delta\), and coupling \(r\) used in the theorem and solver are Schur-complement quantities derived from \eqref{eq:2seg_mass_entries_model}.  The determinant \(\Delta\) is the regularity denominator of the carrier block, \(\Meff\) is the effective inertia seen by the internal rotor after the carrier momentum has been closed, and \(r\) is the yaw-propulsion coupling that reappears when the carrier channel is opened.  They are evaluated explicitly in Section~\ref{sec:2seg_numerics}; their role is to make the transverse storage \(\Eperp=\frac12\Meff(\delta)\sigma^2+U(\delta)\) the closed-channel oscillator energy.

\section{Exact Scalar Closure in the One-Effective-Degree 2SEG}
\label{sec:2seg_theorem}

The finite-speed theorem uses a reduced opened-channel 2SEG chart.  For prescribed nonzero mean speed \(\vbar\), introduce
\begin{equation}
  h=|\vbar|^{-1},\qquad
  s=\sgn(\vbar),\qquad
  \nu=hv,\qquad
  \qy=v\omega .
\label{eq:finite_speed_vars}
\end{equation}
The oriented section variables are
\begin{equation}
\begin{gathered}
  Y=(\delta,\zeta,u,Q),\qquad \zeta>0,\\
  \sigma=s\zeta,
  \qquad \nu=su,
  \qquad \qy=sQ .
\end{gathered}
\label{eq:oriented_vars}
\end{equation}
The physical yaw velocity is recovered as \(\omega=hQ/u\) whenever \(u\neq0\).  The exact opened rows are
\begin{equation}
\dot\delta=\sigma,\qquad
\dot\sigma=F_\sigma,\qquad
\dot\nu=F_\nu,\qquad
\dot\qy=F_{\qy},
\label{eq:opened_rows}
\end{equation}
where the transverse row is
\begin{equation}
F_\sigma=
\frac{r(\delta)\{\qy-\rho(\delta)(\sigma+\omega)^2\}-U'(\delta)}
{\Meff(\delta)} .
\label{eq:Fsigma}
\end{equation}
The yaw and speed rows \(F_\nu,F_{\qy}\) are the exact finite-speed transport rows of the reduced 2SEG model; they are evaluated on the same regular annulus as the storage coefficients.  The theorem only needs the first return and the storage identity, but the return itself is generated by \eqref{eq:opened_rows}.

The generalized transverse source is
\begin{equation}
\begin{aligned}
  \Fex^{\rm ex}
  &=\Meff(\delta)\dot\sigma+\frac12\Meff'(\delta)\sigma^2+U'(\delta)\\
  &=r(\delta)\{\qy-\rho(\delta)(\sigma+\omega)^2\}
  +\frac12\Meff'(\delta)\sigma^2,
\end{aligned}
\label{eq:Fex_exact}
\end{equation}
Thus \(\Fex^{\rm ex}\) is an exchange source, not an external source: it is generated by the opened ideal-constraint coupling and redistributes conserved mechanical energy between oscillator and propulsion coordinates.  The exact POE power is the storage derivative
\begin{equation}
  \Ppoe^{\rm ex}(t)=\sigma(t)\Fex^{\rm ex}(t)
  =\frac{\dd}{\dd t}\Eperp(\delta(t),\sigma(t)).
\label{eq:Ppoe_exact}
\end{equation}
Let \(T(Y_0)\) be the first positive return time to the oriented section \(\delta=0\), with
\begin{equation}
  Y^+(Y_0)=(0,\zeta^+,u^+,Q^+).
\label{eq:return_map}
\end{equation}
The opened exchange-return residual is
\begin{equation}
  \Rex(Y_0)=
  \begin{bmatrix}
  \Psi_{\rm POE}^{\rm ex}(Y_0)\\ C_u(Y_0)\\ C_Q(Y_0)\\ S(Y_0)
  \end{bmatrix},
\label{eq:Rex_def}
\end{equation}
with rows
\begin{subequations}
\label{eq:Rex_rows}
\begin{align}
  \Psi_{\rm POE}^{\rm ex}(Y_0)
  &=\int_0^{T(Y_0)}P_{\rm POE}^{\rm ex}(t;Y_0)\dd t,\\
  C_u(Y_0)&=u(T(Y_0);Y_0)-u(0;Y_0),\\
  C_Q(Y_0)&=Q(T(Y_0);Y_0)-Q(0;Y_0),\\
  S(Y_0)&=\frac{1}{T(Y_0)}\int_0^{T(Y_0)}u(t;Y_0)\dd t-1 .
\end{align}
\end{subequations}
The first row is not a diagnostic added after shooting.  By \eqref{eq:Ppoe_exact},
\begin{equation}
  \Psi_{\rm POE}^{\rm ex}(Y_0)
  =\Eperp(0,s\zeta^+)-\Eperp(0,s\zeta_0)
  =E_\Sigma(\zeta^+)-E_\Sigma(\zeta_0),
\label{eq:psi_energy_jump}
\end{equation}
where on the section
\begin{equation}
  E_\Sigma(\zeta)=\Eperp(0,s\zeta)
  =\frac12\Meff(0)\zeta^2+U(0).
\label{eq:ESigma}
\end{equation}
If \(\Meff(0)>0\) and \(\zeta>0\), then
\begin{equation}
  E_\Sigma'(\zeta)=\Meff(0)\zeta>0.
\label{eq:ESigma_monotone}
\end{equation}
Thus equality of section storage is equality of the positive crossing coordinate.  This is the scalar mechanism of the theorem.

\begin{assumption}[Regular exchange-return annulus]
\label{ass:regular_annulus}
For fixed \(\vbar\neq0\), the considered annulus in \(\Sigmaori\) has a \(C^1\) first return generated by \eqref{eq:opened_rows}; \(\Meff(\delta)>0\), the finite-speed denominators remain nonzero along the return, the initial and terminal crossings satisfy \(\zeta_0>0\) and \(\zeta^+>0\), and the speed row \(S\) is well defined.
\end{assumption}

\begin{theorem}[2SEG POE--NLM equivalence]
\label{thm:2seg_equiv}
Under Assumption~\ref{ass:regular_annulus}, a section point \(Y_0\in\Sigmaori\) represents a 2SEG NLM cycle at mean speed \(\vbar\) if and only if
\begin{equation}
  \Rex(Y_0)=0.
\label{eq:Rex_zero}
\end{equation}
Equivalently,
\begin{equation}
\begin{gathered}
  Y^+(Y_0)=Y_0\quad\hbox{with the mean-speed normalization}\\
  \Longleftrightarrow\\
  \Psi_{\rm POE}^{\rm ex}=0,\quad C_u=0,\quad C_Q=0,\quad S=0 .
\end{gathered}
\label{eq:theorem_equiv_statement}
\end{equation}
\end{theorem}

\begin{proof}
If \(Y_0\) is a section fixed point with the prescribed mean-speed normalization, then \(u^+=u_0\), \(Q^+=Q_0\), \(S=0\), and \(\zeta^+=\zeta_0\).  Equation~\eqref{eq:psi_energy_jump} gives \(\Psi_{\rm POE}^{\rm ex}=0\), hence \(\Rex(Y_0)=0\).

Conversely suppose \(\Rex(Y_0)=0\).  The equations \(C_u=0\) and \(C_Q=0\) give \(u^+=u_0\) and \(Q^+=Q_0\).  The row \(S=0\) fixes the mean-speed normalization.  The exchange row gives
\begin{equation}
\begin{gathered}
  \Psi_{\rm POE}^{\rm ex}=0
  \Rightarrow
  \Eperp(T;Y_0)=\Eperp(0;Y_0)\\
  \Rightarrow
  E_\Sigma(\zeta^+)=E_\Sigma(\zeta_0)
  \Rightarrow
  \zeta^+=\zeta_0,
\end{gathered}
\label{eq:energy_equal_proof}
\end{equation}
where the last implication uses strict monotonicity on the positive oriented section.  The return is already to \(\delta=0\), so all section coordinates return:
\begin{equation}
  Y^+=(0,\zeta^+,u^+,Q^+)=(0,\zeta_0,u_0,Q_0)=Y_0 .
\label{eq:section_return_proof}
\end{equation}
Thus the opened trajectory is an internal section fixed point at the prescribed mean speed.  The POE row is the missing scalar transverse return condition, not an auxiliary residual.
\end{proof}

The proof explains why the 2SEG statement does not generalize as a scalar sufficiency theorem.  If the transverse section coordinate were multidimensional, equality of one scalar storage value would constrain the return to a level set rather than identify the returned state.  The oriented scalar chart converts storage equality into coordinate equality.  This is the mathematical reason the 2SEG statement is an exact equivalence while the 3SEG statement must be modal.

\section{Numerical Realization of the 2SEG Scalar Return}
\label{sec:2seg_numerics}

The scalar computation follows the same mechanical order.  The propulsion channel is first closed, producing a one-dimensional transverse oscillator whose energy level gives a regular support.  That support is used as a chart and seed.  Certification occurs only after the yaw-propulsion channel has been opened, the exact finite-speed return has been integrated, and the exchange rows have been assembled on the returned cycle.

The numerical input is
\begin{equation}
  (\theta_{\rm phys},\vbar),\qquad \vbar\ne0,
\label{eq:numerical_input}
\end{equation}
where \(\theta_{\rm phys}=(\varepsilon,\gamma,k_2,k_4)\) fixes the dimensionless 2SEG mechanism.  The scalar potential is
\begin{equation}
  U(\delta)=\frac12k_2\delta^2+\frac14k_4\delta^4,
  \qquad U'(\delta)=k_2\delta+k_4\delta^3 .
\label{eq:potential}
\end{equation}
The storage and yaw-coupling coefficient block is evaluated at every quadrature or ODE step.  With
\begin{equation}
\begin{gathered}
  \alpha=\varepsilon,\quad \beta=1-\varepsilon,\quad \mu=1-\varepsilon,\\
  j_1=\gamma\alpha^3,\quad j_2=\gamma\beta^3,
  \quad B_{22}=j_2+\mu\beta^2,
\end{gathered}
\label{eq:design_constants}
\end{equation}
the storage coefficients are
\begin{equation}
\begin{aligned}
 B_{11}(\delta)&=j_1+j_2+\alpha^2+\mu\beta^2
 +2\mu\alpha\beta\cos\delta,\\
 B_{12}(\delta)&=j_2+\mu\beta^2+\mu\alpha\beta\cos\delta,\\
 \rho(\delta)&=\mu\beta\sin\delta,
 \qquad
 \Delta(\delta)=B_{11}(\delta)-\rho(\delta)^2.
\end{aligned}
\label{eq:storage_coefficients}
\end{equation}
The scalar transverse inertia is the Schur complement
\begin{equation}
  \Meff(\delta)=B_{22}-
  \frac{B_{12}(\delta)^2+\rho(\delta)^2[B_{11}(\delta)-2B_{12}(\delta)]}{\Delta(\delta)}.
\label{eq:meff_num}
\end{equation}
The domain acceptance conditions require
\begin{equation}
  \Delta(\delta)\ne0,
  \qquad
  \Meff(\delta)>0,
\label{eq:domain_conditions}
\end{equation}
on every sampled support and integrated opened trajectory.

The derivative of the scalar inertia and the yaw-propulsion coefficients are assembled from the same evaluator,
\begin{subequations}
\label{eq:2seg_num_aux_coeffs}
\begin{align}
  \Meff'&=-\frac{N_M'\Delta-N_M\Delta'}{\Delta^2},\\
  N_M&=B_{12}^2+\rho^2(B_{11}-2B_{12}),\\
  A_y(\delta)&=\alpha j_2+\alpha\beta^2\mu(1-\mu)-\mu\beta j_1\cos\delta,\\
  r(\delta)&=\frac{A_y(\delta)}{\Delta(\delta)},\\
  G_y(\delta)&=\alpha+\mu\beta\cos\delta .
\end{align}
\end{subequations}
These terms are not post-processing observables.  They are the coefficients used simultaneously by the closed storage quadrature, the opened finite-speed vector field, and the exact POE accumulator.

The closed-channel support is sampled through the regular phase coordinate
\begin{equation}
  \delta=A\sin\theta,
  \qquad \theta\in[0,2\pi),
\label{eq:theta_param}
\end{equation}
which avoids singular endpoint sampling at turning points.  With midpoint samples \(\theta_j=(j+1/2)\Delta\theta\),
\begin{equation}
\begin{aligned}
  \delta_j&=A\sin\theta_j,\\
  \sigma_j&=\operatorname{sgn}(\cos\theta_j)
  \left[\frac{2(\Ep-U(\delta_j))}{\Meff(\delta_j)}\right]^{1/2},\\
  \Delta t_j&=\frac{|A\cos\theta_j|\Delta\theta}{|\sigma_j|}.
\end{aligned}
\label{eq:support_discretization}
\end{equation}
The support defect retained for diagnostics is
\begin{equation}
  R_{\rm supp}(\Ep)=
  \max_j\left|\frac12\Meff(\delta_j)\sigma_j^2+U(\delta_j)-\Ep\right|.
\label{eq:R_support}
\end{equation}
The support constructs the scalar closed oscillator and provides a structured seed.  It does not certify the final locomotion cycle.

For a nonzero prescribed mean speed, the opened return is integrated in the finite-speed variables of Section~\ref{sec:2seg_theorem}.  The auxiliary rows used in the exact transport are
\begin{subequations}
\label{eq:2seg_num_aux_rows}
\begin{align}
\lambda_{\natural}
 &=\frac{G_y}{\Delta}(\nu-2h\rho\sigma),\\
\phi_{\natural}
 &=\nu\left[-\frac{B_{12}-\rho^2}{\Delta}F_\sigma
 +\frac{\rho G_y}{\Delta}\sigma^2\right],\\
E_{\rm geom}
 &=B_{11}(G_y-\alpha)+\alpha\rho^2,\\
R_{\natural}
 &=\frac{\rho(B_{11}-B_{12})F_\sigma
 +E_{\rm geom}(2\omega\sigma+\sigma^2)
 +B_{11}G_y\omega^2}{\Delta}.
\end{align}
\end{subequations}
The yaw-rate product and speed rows are
\begin{subequations}
\label{eq:2seg_num_Fq_Fnu}
\begin{align}
F_{\qy}
 &=\frac{\phi_{\natural}+h^2\nu^{-1}\qy R_{\natural}
 -\lambda_{\natural}\qy}{h},\\
F_\nu
&=\frac{ h\rho F_\sigma+h(G_y-\alpha)\sigma^2
 +h^2\rho\nu^{-1}F_{\qy}}
 {1+h^2\rho\nu^{-2}\qy} \notag\\
&\quad
+\frac{2h^2(G_y-\alpha)\nu^{-1}\qy\sigma
 +h^3G_y\nu^{-2}\qy^2}
 {1+h^2\rho\nu^{-2}\qy}.
\end{align}
\end{subequations}
In oriented time \(\vartheta=s t\), the integrated system is
\begin{equation}
  \frac{\dd Y}{\dd\vartheta}
  =\begin{bmatrix}\zeta & F_\sigma & F_\nu & F_{\qy}\end{bmatrix}^{\mathsf T}_{(\delta,s\zeta,su,sQ)} .
\label{eq:2seg_num_oriented_flow}
\end{equation}
The trajectory is rejected before residual solving if it violates \(\nu\ne0\), \(h\ne0\), or \(1+h^2\rho\nu^{-2}\qy\ne0\), or if it fails to return to the oriented section with \(\zeta>0\).

For fixed \(\vbar\), the section unknown is parametrized by
\begin{equation}
  Y_0(a,b,Q_0)=(0,e^a,e^b,Q_0),
\label{eq:solve_variables}
\end{equation}
which enforces \(\zeta_0>0\) and \(u_0>0\).  The scalar return is solved as
\begin{equation}
  (a^*,b^*,Q_0^*)=
  \arg\min_{a,b,Q_0}\frac12
  \norm{\widehat\Rex(Y_0(a,b,Q_0);\vbar)}^2,
\label{eq:least_squares_problem}
\end{equation}
where \(\widehat\Rex=W\Rex\) is an invertibly scaled residual.  The scaling improves conditioning but does not change the zero set,
\begin{equation}
  \widehat\Rex=0\quad\Longleftrightarrow\quad\Rex=0 .
\label{eq:scaled_equiv}
\end{equation}

The unscaled rows remain mechanical diagnostics.  The POE and mean-speed rows are accumulated during the event integration,
\begin{subequations}
\label{eq:2seg_num_accumulators}
\begin{align}
  \frac{\dd I_{\POE}}{\dd\vartheta}&=s\,\Ppoe^{\rm ex}(Y(\vartheta)),
  & I_{\POE}(0)&=0,\\
  \frac{\dd I_u}{\dd\vartheta}&=u(\vartheta),
  & I_u(0)&=0.
\end{align}
\end{subequations}
At the return time \(\tau\),
\begin{equation}
  \Rex(Y_0;\vbar)=
  \begin{bmatrix}
  I_{\POE}(\tau)\\
  u^+-u_0\\
  Q^+-Q_0\\
  I_u(\tau)/\tau-1
  \end{bmatrix}.
\label{eq:2seg_num_Rex_full}
\end{equation}
A typical scaling is
\begin{subequations}
\label{eq:2seg_num_scaling_rows}
\begin{align}
  W&={\rm diag}(e_{\rm sc}^{-1},u_{\rm sc}^{-1},Q_{\rm sc}^{-1},1),\\
  e_{\rm sc}&=\max(1,|I_{\POE}|,|\Delta\Eperp|),\\
  u_{\rm sc}&=\max(1,|u_0|),\qquad Q_{\rm sc}=\max(1,|Q_0|).
\end{align}
\end{subequations}
The weighted solve is thus a conditioning device; the certificate is the unscaled exchange-return residual.

The accepted orbit is
\begin{equation}
  \Gamma_{\vbar}=\{Y(\vartheta;Y_0^*,\vbar):0\le\vartheta\le\tau(Y_0^*)\}.
\label{eq:accepted_orbit}
\end{equation}

The acceptance conditions are separated by role:
\begin{equation}
\begin{aligned}
\mathcal G_{\rm dom}:&\quad
\Delta\ne0,\ \Meff>0,\ \nu\ne0,\ 1+h^2\rho\nu^{-2}\qy\ne0,\\
\mathcal G_{\rm sec}:&\quad
\tau<\infty,\ \delta(\tau)=0,\ \zeta(\tau)>0,\\
\mathcal G_{\rm res}:&\quad
\norm{\widehat\Rex(Y_0^*;\vbar)}\le\tau_{\rm ex},\\
\mathcal G_{\rm diag}:&\quad
\mathbf r_{\rm phys}\ \text{reported.}
\end{aligned}
\label{eq:condition_table_2seg}
\end{equation}
Continuation in mean speed uses a solved point only as a predictor; validity is not continued.  Each new speed is re-certified by the exact opened residual.  After internal acceptance, the pose is reconstructed as an output,
\begin{equation}
  \dot g=g\,\xi(Y(t);\theta_{\rm phys},\vbar),
  \qquad
  \Delta g_{\vbar}=g(\tau)g(0)^{-1}.
\label{eq:pose_reconstruction_2seg}
\end{equation}

\begin{construction}[Exact scalar 2SEG numerical realization]
For fixed \(\theta_{\rm phys}\) and \(\vbar\ne0\): evaluate storage, yaw, and exchange coefficients; build the carrier-closed chart; convert a structured seed to oriented finite-speed variables; integrate the opened flow to the next positive section event while accumulating POE and mean-speed integrals; assemble \(\Rex\), solve the scaled problem, and retain unscaled diagnostics; accept only if domain, section, and residual conditions all hold; warm-start in \(\vbar\) only as prediction and re-certify every reported point.
\end{construction}

The reported observables are also separated by role:
\begin{equation}
\begin{gathered}
\text{construction: } A,\ \Ep,\ T_L,\ J_L,\ Y_0^*,\ \tau,\\
\text{certificate: } I_{\POE},\ C_u,\ C_Q,\ S,\ \norm{\widehat\Rex},\ \mathbf r_{\rm phys},\\
\text{locomotion: } \vbar,\ \Delta g_{\vbar}.
\end{gathered}
\label{eq:2seg_num_observables}
\end{equation}
Construction coordinates describe the closed oscillator and the section point; certificate rows accept or reject the opened return; locomotion observables are reported only after the internal exchange-return certificate closes.

The rows are ordered by the closure they enforce.  The carrier rows close the opened yaw-propulsion variables, the mean-speed row fixes the normalized mean speed, and the POE row closes the last scalar transverse storage coordinate.  The result is the exact one-effective-coordinate numerical realization of the scalar natural-locomotion branch.

\section{3SEG Mechanical Model and Transverse Split}
\label{sec:3seg_model_primer}

The 3SEG realization uses a planar pendulum-driven car with a serial two-joint internal oscillator.  Body 1 is the nonholonomic base/carrier, body 2 is the first internal link, and body 3 is the second internal link.  The configuration is
\begin{equation}
  q=(x,y,\theta,\delta_1,\delta_2),
  \qquad \sigma_i=\dot\delta_i .
\label{eq:3seg_config_model}
\end{equation}
The pair \((x,y,\theta)\) is the carrier pose, \(\delta_1,\delta_2\) are internal yaw angles, and \(\sigma_1,\sigma_2\) are their rates.
The no-side-slip constraint at the contact point is
\begin{equation}
  A(q)\dot q=0,
  \qquad
  A(q)=(-\sin\theta,\cos\theta,0,0,0).
\label{eq:3seg_pfaff_model}
\end{equation}
A basis of admissible velocities gives
\begin{equation}
  \dot q=N(q)\nu,
  \qquad
  \nu=(v,\omega,\sigma_1,\sigma_2)^{\mathsf T},
\label{eq:3seg_kernel_model}
\end{equation}
so that reconstruction is
\begin{equation}
  \dot x=v\cos\theta,
  \quad \dot y=v\sin\theta,
  \quad \dot\theta=\omega,
  \quad \dot\delta_i=\sigma_i .
\label{eq:3seg_reconstruction_model}
\end{equation}
The physical parameters used by the model are summarized before the formulas because the same symbols later define the fixed architecture \(\thetaStar\).
\begin{table}[!t]
\caption{Reader-facing 3SEG parameter definitions.  The model is strictly conservative when the torque inputs \(u_1,u_2\) are set to zero.}
\label{tab:3seg_parameter_definitions}
\centering
\scriptsize
\setlength{\tabcolsep}{2.2pt}
\renewcommand{\arraystretch}{1.06}
\begin{tabularx}{\columnwidth}{@{}p{0.25\columnwidth}X@{}}
\toprule
Symbol & Mechanical meaning \\
\midrule
\(m_i,I_i\) & Mass and planar center-of-mass inertia of body \(i\in\{1,2,3\}\). \\
\(l_1\) & Distance from contact point to the base center of mass and first joint along the base axis. \\
\(L_2,L_3\) & Internal link lengths; the slender-link convention uses center-of-mass offsets \(l_2=L_2/2\) and \(l_3=L_3/2\). \\
\(k_{2,i},k_{4,i}\) & Quadratic and quartic passive torsional stiffnesses at internal joint \(i\). \\
\(k_{12}\) & Passive coupling stiffness between the two internal angles. \\
\(u_1,u_2\) & Generalized torque inputs associated with \(\delta_1,\delta_2\); the conservative certificates use \(u_1=u_2=0\), not tuned offsets. \\
\bottomrule
\end{tabularx}
\end{table}

The conservative internal potential is
\begin{equation}
\begin{aligned}
V(\delta_1,\delta_2)
&=\tfrac12 k_{2,1}\delta_1^2+\tfrac14 k_{4,1}\delta_1^4
+\tfrac12 k_{2,2}\delta_2^2+\tfrac14 k_{4,2}\delta_2^4\\
&\quad +\tfrac12 k_{12}(\delta_1-\delta_2)^2 .
\end{aligned}
\label{eq:3seg_potential_model}
\end{equation}
After kernel projection of the Lagrange--d'Alembert equations, the reduced dynamics is
\begin{equation}
  \bar M(\delta)\dot\nu+\bar h(\delta,\nu)+\bar g(\delta)=0,
  \qquad
  \bar g=(0,0,\partial_{\delta_1}V,\partial_{\delta_2}V)^{\mathsf T},
\label{eq:3seg_reduced_model}
\end{equation}
with conservative energy
\begin{equation}
  E(\delta,\nu)=\tfrac12\nu^{\mathsf T}\bar M(\delta)\nu+V(\delta),
  \qquad \dot E=0 .
\label{eq:3seg_energy_model}
\end{equation}
The carrier velocities are \((v,\omega)\) and the shape velocities are \((\sigma_1,\sigma_2)\).  Therefore \(\bar M\) is partitioned as
\begin{equation}
\bar M(\delta)=
\begin{bmatrix}
M_{GG}(\delta)&M_{GS}(\delta)\\
M_{SG}(\delta)&M_{SS}(\delta)
\end{bmatrix},
\label{eq:3seg_block_model}
\end{equation}
leading to the closed-channel Schur complement
\begin{equation}
  \Mperp=M_{SS}-M_{SG}M_{GG}^{-1}M_{GS},
  \qquad
  \eta_\perp=-M_{GG}^{-1}M_{GS}\sigma .
\label{eq:3seg_schur_model}
\end{equation}
The Schur complement removes the carrier momentum contribution from the energy while retaining the shape-dependent inertia that the internal links feel.  This is the mechanical origin of the transverse oscillator, the in-phase/anti-phase modal supports, and the final POE row.  The frozen passive design vector used in the same-physical result is
\begin{equation}
\begin{aligned}
\phys=(&I_1,I_2,I_3,l_1,L_2,L_3,m_1,m_2,m_3,\\
&k_{12},k_{2,1},k_{2,2},k_{4,1},k_{4,2}).
\end{aligned}
\label{eq:3seg_phys_vector_model}
\end{equation}
The conservative opened-cycle certificates are evaluated with \(u_1=u_2=0\).  These \(u_i\) are torque inputs, not passive offsets and not family-specific tuning variables.

\section{3SEG Modal Numerical Realization}
\label{sec:3seg_realization}

\subsection{Candidate Objects and Final Certificates}
\label{subsec:3seg_candidate_objects}

This section defines the numerical objects and separates candidate constructions from final mechanical NLM certificates.  The 3SEG construction follows the same close/open principle as the scalar case, but the closed object is no longer selected by one energy level.  The carrier-closed system has two effective internal shape directions, so the computation has two roles: it must generate enough nonlinear transverse modal candidates to expose IP/AP structure, and it must accept only the opened reconstructed cycles that satisfy the NLM certificate.  Support solves, collocation, multiple shooting, continuation charts, candidate mobility tests, covers, and time reversal are candidate-generation tools.  A candidate mobility test checks whether a candidate can be transported toward a nonzero-displacement continuation; the reported point is the reconstructed cycle that closes the final certificate.

For a fixed physical vector,
\begin{equation}
\begin{aligned}
\phys &\longrightarrow f_{\rm int}
\longrightarrow f_\perp
\longrightarrow (u_{\IP},u_{\AP})
\longrightarrow \gamma_{\perp,k}\\
&\longrightarrow z_k
\longrightarrow \Gamma_{k,\vbar}
\longrightarrow \calC_k^{\rm fin}(\Gamma_{k,\vbar};\phys)
\longrightarrow \Delta g_{k,\vbar},
\end{aligned}
\label{eq:3seg_chain}
\end{equation}
with
\begin{equation}
 k\in\calK_{\rm 3SEG}:=\{\IP,\AP\}.
\label{eq:sector_set}
\end{equation}
The label \(k\) denotes a nonlinear modal sector of the closed transverse oscillator; it is not a continuation-table label and not a full-state nonlinear normal mode.

Candidate equations and final certificates are distinct.  A chart \(c\in\calC_k^{\rm chart}\) has unknowns \(x_k^{(c)}\) and residual
\begin{equation}
\calR_k^{{\rm cand},c}(x_k^{(c)};\phys)=0.
\label{eq:candidate_residual}
\end{equation}
Its scaled form,
\begin{equation}
\widehat{\calR}_k^{{\rm cand},c}=W_k^{(c)}\calR_k^{{\rm cand},c},
\label{eq:candidate_scaling}
\end{equation}
conditions the chart solve.  A chart residual may contain phase gauges, speed targets, modal floor rows, branch-tangent rows, or homotopy rows.  These rows define a representation of a candidate; they are not by themselves mechanical acceptance conditions.

After reconstruction, the certificate is separated into representation checks and mechanical checks.  Here and below, BVP abbreviates boundary-value problem: the corresponding row measures the consistency of a chosen discrete cycle representation, not the continuous mechanical dynamics.  Representation rows verify that the selected numerical representation is internally consistent,
\begin{equation}
\calC_k^{{\rm rep},c}=\bigl(R_{\rm ph}^{(c)},R_{\rm ms}^{(c)},R_{\rm BVP}^{(c)},R_{\vbar}^{(c)},R_{\rm phys}^{(c)}\bigr)\transpose,
\label{eq:rep_certificate}
\end{equation}
with chart-dependent rows omitted when they are not active.  The mechanical final certificate is evaluated from the reconstructed cycle,
\begin{equation}
\calC_k^{\rm mech}(\Gamma)=
\bigl(R_{\rm supp}^{\rm fin},R_{\rm id}^{\rm fin},R_{\rm car}^{\rm fin},R_z^{\rm fin},R_{\rm rhs}^{\rm fin},R_{\rm POE}^{\rm fin},R_g^{\rm fin}\bigr)\transpose.
\label{eq:mech_certificate}
\end{equation}
The reported certificate is
\begin{equation}
\calC_k^{\rm fin}=\begin{bmatrix}
\calC_k^{{\rm rep},c}\\
\calC_k^{\rm mech}
\end{bmatrix},
\label{eq:final_certificate_tuple}
\end{equation}
but the mechanical claims use \eqref{eq:mech_certificate}.  The support-return row may be evaluated either on the transverse support or on the reconstructed projection,
\begin{equation}
R_{\rm supp}^{\rm fin}=
\left\|D_y^{-1}\bigl(\Pi_y\Gamma(T)-\Pi_y\Gamma(0)\bigr)\right\|,
\label{eq:Rsupp_final}
\end{equation}
with \(\Pi_y z=(r,\sigma)\).  The nonzero locomotion row is
\begin{equation}
R_g^{\rm fin}=\max\{0,d_{\min}-d_g(\Delta g,e)\}.
\label{eq:Rg_final}
\end{equation}
Scaled final rows may be used for acceptance,
\begin{equation}
\widehat{\calC}_k^{\rm fin}=W_k^{\rm fin}\calC_k^{\rm fin},
\qquad
\widehat{\calC}_k^{\rm mech}=W_k^{\rm mech}\calC_k^{\rm mech},
\label{eq:final_scalings}
\end{equation}
where the scaling matrices are nonsingular on active rows.  Thus scaling improves conditioning but does not alter the zero set,
\begin{equation}
\widehat{\calC}_k^{\rm mech}=0
\quad\Longleftrightarrow\quad
\calC_k^{\rm mech}=0.
\label{eq:scaling_equivalence}
\end{equation}
Acceptance is row-wise and/or normed,
\begin{equation}
R_i^{\rm fin}\le \tau_i,
\qquad
\norm{\widehat{\calC}_k^{\rm fin}}\le\tau_k,
\label{eq:acceptance}
\end{equation}
while unscaled rows remain the mechanical diagnostics.  Speed and physical-parameter rows are conditional representation rows.  In speed-target charts,
\begin{equation}
R_{\vbar}^{(c)}=\frac{\vbar(\Gamma)-\vbar_{\rm tar}}{v_{\rm sc}},
\label{eq:Rbarv_conditional}
\end{equation}
whereas charts that only report speed omit the row and recompute \(\vbar(\Gamma)\) after certification.  In physical homotopies,
\begin{equation}
R_{\rm phys}^{(c)}=\norm{(\phys-\theta_{\rm tar})/\theta_{\rm sc}},
\label{eq:Rphys_conditional}
\end{equation}
whereas fixed-physical charts take \(\phys\) as prescribed.

\subsection{Closed Transverse Supports and Modal Gauges}
\label{sec:closed_supports}

This section constructs the IP/AP modal vocabulary in the carrier-closed internal system.  These supports organize the internal oscillator, but they do not yet assert locomotion because the propulsive carrier is closed.

The reference dynamics is the six-state internal field
\begin{equation}
z=(\delta_1,\delta_2,v,\omega,\sigma_1,\sigma_2),
\qquad
\dot z=f_{\rm int}(z;\phys).
\label{eq:rhs_internal_consistency}
\end{equation}
The transverse oscillator is obtained by closing the carrier.  With \(y=(r,\sigma)\), \(\eta=(v,\omega)\), and
\begin{equation}
M(r)=
\begin{bmatrix}
M_{GG}(r)&M_{GS}(r)\\
M_{SG}(r)&M_{SS}(r)
\end{bmatrix},
\label{eq:mass_split}
\end{equation}
define
\begin{subequations}
\label{eq:schur_layer}
\begin{align}
A(r)&=M_{GG}^{-1}(r)M_{GS}(r),\\
\eta_\perp&=-A(r)\sigma,\\
\Mperp(r)&=M_{SS}-M_{SG}M_{GG}^{-1}M_{GS},\\
\Eperp(r,\sigma)&=\frac12\sigma\transpose\Mperp(r)\sigma+U(r).
\end{align}
\end{subequations}
The carrier-closed field is
\begin{equation}
\dot y=f_\perp(y;\phys),
\qquad
\frac{\dd}{\dd t}\Eperp(y(t))=0,
\label{eq:transverse_field}
\end{equation}
and accepted supports remain inside the regular transverse domain,
\begin{equation}
\lambda_{\min}\bigl(\Mperp(r(t))\bigr)\ge m_{\min}>0.
\label{eq:Mperp_pd}
\end{equation}

The IP/AP parents are initialized by the transverse linearization
\begin{equation}
K_0u_j=\Omega_j^2M_0u_j,
\qquad
u_i\transpose M_0u_j=\delta_{ij},
\qquad j\in\{\IP,\AP\},
\label{eq:parent_modes}
\end{equation}
where \(M_0=\Mperp(0)\), \(K_0=D^2U(0)\).  Modal coordinates are
\begin{equation}
Q_j=u_j\transpose M_0r,
\qquad
P_j=u_j\transpose M_0\sigma,
\qquad
\widetilde P_j=P_j/\Omega_j.
\label{eq:modal_coords}
\end{equation}
A pure small-amplitude seed uses
\begin{equation}
r_0=A_j u_j,
\qquad
\sigma_0=0,
\qquad
T_0=2\pi/\Omega_j.
\label{eq:spectral_seed}
\end{equation}
The finite-amplitude support is then solved as a periodic orbit of \(f_\perp\):
\begin{equation}
F_{\rm supp}(y_0,T)=\Phi_\perp^T(y_0;\phys)-y_0=0.
\label{eq:support_periodicity}
\end{equation}
Because the orbit is phase-invariant, a gauge is required.  A tangent gauge is
\begin{equation}
\Phi_{\rm ph}^{\rm tan}(y_0)=
\left\langle y_0-y_{\rm ref},\dot y_{\rm ref}\right\rangle_{D_y^{-2}}=0.
\label{eq:tangent_phase}
\end{equation}
A modal gauge for sector \(j\) is
\begin{equation}
P_j(y_0)=0,
\qquad
Q_j(y_0)>0,
\label{eq:modal_phase}
\end{equation}
and a seed-tracking gauge is
\begin{equation}
\Phi_{\rm ph}^{\rm seed}(y_0,\phi)=
\frac{\langle y_0-y_{\rm ref}(\phi),\dot y_{\rm ref}(\phi)\rangle}
{\norm{\dot y_{\rm ref}(\phi)}^2+\varepsilon_{\rm ph}}=0.
\label{eq:seed_phase}
\end{equation}
Pure transverse supports use these transverse gauges.  Internal/BVP charts may instead use a blended gauge after the support has been lifted:
\begin{equation}
\begin{aligned}
\Phi_{\rm ph}^{(j)}(z_0,\phi)={}&
\alpha\frac{v_0-v_{\rm ref}(\phi)}{\max(|v_{\rm ref}(\phi)|,v_{\rm sc})}\\
&+(1-\alpha)\frac{P_j(z_0)-P_j(z_{\rm ref}(\phi))}{\max(|P_j(z_{\rm ref}(\phi))|,p_{\rm sc})}=0.
\end{aligned}
\label{eq:blended_phase}
\end{equation}

A pure support solve is
\begin{equation}
\calR_{\perp,j}^{\rm pure}(y_0,T,A_j)=
\begin{bmatrix}
D_y^{-1}F_{\rm supp}(y_0,T)\\[0.15em]
P_j(y_0)/p_{\rm sc}\\[0.15em]
(Q_j(y_0)-A_j)/q_{\rm sc}
\end{bmatrix}=0.
\label{eq:pure_support}
\end{equation}
A mixed chart activates only the modal rows required by that chart.  Let
\begin{equation}
m(y_0)=(Q_{\IP},P_{\IP},Q_{\AP},P_{\AP})(y_0),
\label{eq:modal_row}
\end{equation}
and let \(S_m^{(k,c)}\) select the active components in chart \(c\).  The vector \(\mu\) collects free targets or continuation parameters, so the system may be square or solved in scaled least squares without changing the final certificate.
\begin{equation}
\calR_{\perp,k}^{\rm mix,c}(y_0,T,\mu)=
\begin{bmatrix}
D_y^{-1}F_{\rm supp}(y_0,T)\\[0.15em]
D_m^{-1}S_m^{(k,c)}\bigl(m(y_0)-m_k^{\rm tar}(\mu)\bigr)
\end{bmatrix}=0.
\label{eq:mixed_support}
\end{equation}
Support diagnostics are reported unscaled:
\begin{subequations}
\label{eq:support_diags}
\begin{align}
R_{\rm supp}&=\norm{\Phi_\perp^{T_k}(y_{0,k})-y_{0,k}},\\
R_E&=\max_t|\Eperp(y_k(t))-\Eperp(y_k(0))|,\\
R_{\rm ph}&=|\Phi_{\rm ph}(y_{0,k})|.
\end{align}
\end{subequations}

Modal chart rows are not final modal identity.  Identity is recomputed from the reconstructed cycle.  Activity shares are
\begin{subequations}
\label{eq:modal_activity}
\begin{align}
A_j^2(\Gamma)&=\frac1T\int_0^T(Q_j(t)^2+\widetilde P_j(t)^2)\,\dd t,\\
\rho_j(\Gamma)&=\frac{A_j^2}{A_{\IP}^2+A_{\AP}^2+\varepsilon_{\rm id}},
\qquad \varepsilon_{\rm id}>0.
\end{align}
\end{subequations}
A relative phase feature may be computed as
\begin{equation}
\phi_{\rm rel}=\arg\left(\int_0^T(Q_{\IP}+i\widetilde P_{\IP})(Q_{\AP}-i\widetilde P_{\AP})\,\dd t\right).
\label{eq:rel_phase}
\end{equation}
The final identity row is a distance in a feature space containing activity shares, phase/sign, and correlation information,
\begin{equation}
R_{\rm id}^{(k)}=\calD_{\rm id}^{(k)}(A_{\IP},A_{\AP},\rho_{\IP},\rho_{\AP},\phi_{\rm rel},c_{\rm corr},s_{\rm sign}),
\label{eq:id_condition}
\end{equation}
and sector identity is accepted when \(R_{\rm id}^{(k)}\le\tau_{\rm id}^{(k)}\).

\subsection{Carrier Lift, BVP Representations, and Continuation Charts}
\label{sec:lift_charts}

This section describes the numerical representations used to transport and polish candidates after the carrier degrees of freedom are reopened.  The rows below are numerical handles for candidate motion and conditioning; they are not mechanical claims until the reconstructed cycle is recertified.

The opened internal lift of a transverse support is
\begin{equation}
\begin{aligned}
z_k(t)&=(r_k(t),\eta_k(t),\sigma_k(t)),\\
\eta_k(t)&=-A(r_k(t))\sigma_k(t)+\qck(t),
\end{aligned}
\label{eq:z_lift}
\end{equation}
with carrier dynamics represented by
\begin{equation}
\dot\qck=f_c(y_k(t),\qck(t);\phys).
\label{eq:carrier_rhs}
\end{equation}
The carrier row is
\begin{equation}
R_{\rm car}^{\rm fin}=\norm{\qck(T_k)-\qck(0)}_{D_q^{-1}}.
\label{eq:Rcar}
\end{equation}
A chart-level support-carrier residual is
\begin{equation}
\calR_k^{{\rm cand},c}(x_k^{\rm sc})=
\begin{bmatrix}
D_y^{-1}(y(T_k)-y_{0,k})\\[0.15em]
D_q^{-1}(\qck(T_k)-q_{0,k})\\[0.15em]
R_{\vbar}^{(c)}(\Gamma_k)\\[0.15em]
R_{\rm mod}^{(k,c)}(y_{0,k},\Gamma_k)\\[0.15em]
R_{\rm POE}^{(c)}(\Gamma_k)
\end{bmatrix}.
\label{eq:support_carrier_cand}
\end{equation}
The rows in \(\calR_k^{{\rm cand},c}\) generate candidates.  The final certificate recomputes modal identity, POE, carrier return, and internal return on the reconstructed cycle.

Multiple shooting uses nodes \(Z=(z_0,\ldots,z_{m-1})\) and durations \(\Delta t_i\):
\begin{equation}
\begin{gathered}
\calF_{\rm ms}(Z,T)=
\begin{bmatrix}
D_z^{-1}(\Phi_{\rm int}^{\Delta t_0}(z_0)-z_1)\\
\vdots\\
D_z^{-1}(\Phi_{\rm int}^{\Delta t_{m-1}}(z_{m-1})-z_0)\\
\Phi_{\rm ph}(z_0)
\end{bmatrix},\\
R_{\rm ms}=\norm{\calF_{\rm ms}}.
\end{gathered}
\label{eq:multishoot}
\end{equation}
Hermite--Simpson collocation uses, on each mesh interval,
\begin{subequations}
\label{eq:HS}
\begin{align}
z_{i+1/2}&=\frac12(z_i+z_{i+1})+\frac{\Delta t_i}{8}(f_i-f_{i+1}),\\
f_{i+1/2}&=f_{\rm int}(z_{i+1/2};\phys),\\
H_i&=z_{i+1}-z_i-\frac{\Delta t_i}{6}(f_i+4f_{i+1/2}+f_{i+1})=0.
\end{align}
\end{subequations}
The BVP residual is evaluated on the native candidate representation--multiple-shooting nodes, collocation mesh, Hermite--Simpson polynomial, or controlled reconstructed mesh:
\begin{equation}
\begin{aligned}
\calF_{\rm BVP}&=(H_0,\ldots,H_{N-1},z_N-z_0,\Phi_{\rm ph},\ldots),\\
R_{\rm BVP}^{(c)}&=\norm{\calF_{\rm BVP}}.
\end{aligned}
\label{eq:BVP}
\end{equation}
Thus \(R_{\rm BVP}^{(c)}\) is the discrete representation residual.  The right-hand-side dynamics check is separate: after the candidate has been represented as a reconstructed cycle \(z_{\rm rep}(t)\),
\begin{equation}
R_{\rm rhs}^{\rm fin}=\left\|\dot z_{\rm rep}(t)-f_{\rm int}(z_{\rm rep}(t);\phys)\right\|_\infty,
\label{eq:rhs}
\end{equation}
where \(\dot z_{\rm rep}\) is the polynomial derivative, segment derivative, or controlled-reintegration derivative appropriate to the representation.

Continuation transports candidates through sector-dependent charts.  With unknown vector \(w\), continuation parameter \(\lambda\), and chart \(c\),
\begin{equation}
\begin{bmatrix}
\calF_{\rm cont}^{(k,c)}(w,\lambda)\\[0.15em]
\langle (w,\lambda)-(w_0,\lambda_0),(\dot w_0,\dot\lambda_0)\rangle-\Delta s
\end{bmatrix}=0.
\label{eq:pseudo_arclength}
\end{equation}
Representative active rows are:
\begin{subequations}
\label{eq:chart_rows}
\begin{align}
R_{\vbar}^{\rm tar}&=\frac{\vbar(\Gamma)-\vbar_{\rm tar}}{v_{\rm sc}},\qquad
\vbar(\Gamma)=\frac1T\int_0^T v(t)\,\dd t,\\
R_{\rm POE}^{\rm chart}&=\frac{\Psipoe(\Gamma)}{\psi_{\rm sc}},\\
R_{\rm br}&=\frac{\langle w-w_{\rm br}(\phi),\tau_{\rm br}(\phi)\rangle}{\norm{\tau_{\rm br}(\phi)}^2+\varepsilon_{\rm br}}-\lambda_{\rm br},\\
A_{\rm nonv}(\Gamma)&=\left(\frac1T\int_0^T\norm{P_{\rm nonv}z(t)}^2\,\dd t\right)^{1/2},\\
R_{\rm act}&=\frac{A_{\rm nonv}(\Gamma)-A_{\rm floor}}{a_{\rm sc}},\\
R_{\rm floor}^{(k)}&=\frac{A_k(\Gamma)-A_{\rm floor}^{(k)}}{a_{\rm sc}^{(k)}}.
\end{align}
\end{subequations}
The main chart families are summarized in Table~\ref{tab:chart_rows}.  These rows are numerical handles for transporting candidates through difficult regions.  They can constrain, monitor, or parameterize a candidate chart; final acceptance still uses the reconstructed-cycle certificate.
\begin{table}[t]
\caption{Candidate continuation charts and active rows.}
\label{tab:chart_rows}
\centering
\footnotesize
\begin{tabularx}{\columnwidth}{@{}l l X@{}}
\toprule
Chart & Active row(s) & Numerical role\\
\midrule
mean-speed & \(R_{\vbar}^{\rm tar}\) & target or transport prescribed speed\\
POE-constrained & \(R_{\rm POE}^{\rm chart}\) & constrain POE during candidate generation\\
branch tangent & \(R_{\rm br}\) & parameterize a local branch direction\\
non-\(v\) activity & \(R_{\rm act}\) & avoid inactive candidate regions\\
modal floor & \(R_{\rm floor}^{(k)}\) & maintain sector activity\\
secant & \(R_{\rm sec}^{v}\), \(R_{\rm sec}^{\rm nonv}\) & connect known candidate states\\
physical homotopy & \(R_{\rm phys}^{(c)}\) & move search toward \(\theta_*\)\\
\bottomrule
\end{tabularx}
\end{table}
Secant continuation may use
\begin{equation}
\bar v_{\rm sec}(\lambda)=\bar v_0+\lambda\Delta\bar v,
\qquad
A_{{\rm nonv},{\rm sec}}(\lambda)=A_{{\rm nonv},0}+\lambda\Delta A_{\rm nonv},
\label{eq:secant_rows}
\end{equation}
with residuals \(\bar v(\Gamma)-\bar v_{\rm sec}(\lambda)\) or \(A_{\rm nonv}(\Gamma)-A_{{\rm nonv},{\rm sec}}(\lambda)\).  A physical homotopy chart uses
\begin{equation}
\phys(\alpha)=(1-\alpha)\theta_{\rm src}+\alpha\theta_*,
\qquad \alpha\in[0,1],
\label{eq:physical_homotopy}
\end{equation}
to locate candidates; reported same-physical branches are certified only at \(\phys=\theta_*\).  The equations in this section describe chart rows, not final acceptance rows.

\subsection{Candidate Transformations, Final Acceptance Conditions, and Observables}
\label{sec:final_conditions_obs}

This section states which reconstructed cycles are finally accepted and which observables are reported.  Candidate transformations may expose useful branches, but every reported quantity is recomputed after projection to the accepted opened cycle.

The AP sector may be represented on an auxiliary AP period cover with cover multiplicity \(m_{\rm cov}\):
\begin{subequations}
\label{eq:ap_cover}
\begin{align}
\pi_{\rm cov}&:\widetilde\Gamma_{\AP}\to\Gamma_{\AP},\\
\widetilde\Gamma_{\AP}(t+\widetilde T)&=\widetilde\Gamma_{\AP}(t),\\
T_{\AP}&=\widetilde T/m_{\rm cov},\qquad m_{\rm cov}\in\mathbb N,\\
\Gamma_{\AP}(t)&=\pi_{\rm cov}\widetilde\Gamma_{\AP}(t).
\end{align}
\end{subequations}
Cover-level checks \(\calC_{\AP}^{\rm cov}(\widetilde\Gamma_{\AP})\) may check this auxiliary representation; the reported physical AP cycle is accepted by \(\calC_{\AP}^{\rm fin}(\Gamma_{\AP})\).  The cover is a candidate representation used to expose AP mobility during search; it is not the reported physical cycle.  If the cover is used only to generate a candidate, all reported observables are recomputed on \(\Gamma_{\AP}\).  Time reversal is a candidate map
\begin{equation}
(\calT\Gamma)(t)=\calS_{\rm tr}\Gamma(T-t),
\label{eq:time_reversal}
\end{equation}
where \(\calS_{\rm tr}\) is the model involution that flips the velocity-like components.  Its defining compatibility is
\begin{equation}
D\calS_{\rm tr}(z)\,f_{\rm int}(z;\phys)=-f_{\rm int}(\calS_{\rm tr}z;\phys).
\label{eq:time_reversal_involution}
\end{equation}
A transformed candidate is accepted only if \(\calC_k^{\rm fin}(\calT\Gamma)\) closes.

Candidate mobility probes are normalized scoring maps for candidate regions.  They test whether a candidate can be transported toward nonzero-displacement continuations before any final certificate is asserted:
\begin{equation}
\calB(\phys)=(b_{\rm supp},b_{\rm id},b_{\rm lift},b_{\rm POE},b_{\rm mob}),
\label{eq:candidate_mobility_screen}
\end{equation}
with representative, not unique, normalized components
\begin{equation}
\begin{aligned}
b_{\rm supp}&=e^{-R_{\rm supp}/s_{\rm supp}},\qquad
b_{\rm POE}=e^{-R_{\rm POE}/s_{\rm POE}},\\
b_{\rm mob}&=\mathbf 1_{\{d_g(\Delta g,e)\ge d_{\min}\}}.
\end{aligned}
\label{eq:candidate_mobility_scores}
\end{equation}
Large scores select regions to polish or continue; they do not replace the final certificate.

The POE row is evaluated on the opened reconstructed cycle:
\begin{subequations}
\label{eq:poe}
\begin{align}
\frac{\dd}{\dd t}\Eperp(r(t),\sigma(t))&=\Ppoe(t),\\
\Psipoe(\Gamma_k)&=\frac1{T_k}\int_0^{T_k}\Ppoe(t)\,\dd t,\\
R_{\rm POE}^{\rm fin}&=|\Psipoe(\Gamma_k)|.
\end{align}
\end{subequations}
Candidate charts may be POE-constrained, POE-checked, or POE-free during search; a reported natural locomotion point is POE-certified.  Internal return is
\begin{equation}
R_z^{\rm fin}=\norm{\Pi_z\Gamma_k(T_k)-\Pi_z\Gamma_k(0)}.
\label{eq:Rz}
\end{equation}
Only after the internal certificate closes is the pose reconstructed:
\begin{equation}
\dot g=g\,\xi(z(t);\phys),
\qquad
\Delta g_k=g(T_k)g(0)^{-1},
\label{eq:pose}
\end{equation}
with nonzero locomotion condition \eqref{eq:Rg_final}.

The same-physical IP/AP certificate freezes the physical vector and pairs the final certificates:
\begin{equation}
\calC_{\rm pair}=
\begin{bmatrix}
\widehat{\calC}_{\IP}^{\rm fin}(\Gamma_{\IP};\theta_*)\\[0.15em]
\widehat{\calC}_{\AP}^{\rm fin}(\Gamma_{\AP};\theta_*)\\[0.15em]
(\theta_{\IP}-\theta_*)/\theta_{\rm sc}\\[0.15em]
(\theta_{\AP}-\theta_*)/\theta_{\rm sc}
\end{bmatrix},
\qquad
\norm{\calC_{\rm pair}}\le\tau_{\rm pair}.
\label{eq:paired_certificate}
\end{equation}
Together with nonzero displacement,
\begin{equation}
d_g(\Delta g_{\IP},e)\ge d_{\min},
\qquad
 d_g(\Delta g_{\AP},e)\ge d_{\min},
\label{eq:paired_displacement}
\end{equation}
this gives \(\calN_{\IP}(\theta_*)\ne\varnothing\) and \(\calN_{\AP}(\theta_*)\ne\varnothing\).

The acceptance-condition families separate representation validity from mechanical acceptance:
\begin{subequations}
\label{eq:condition_groups}
\begin{align}
\calG_{\rm dom}:&\quad \lambda_{\min}(\Mperp(r(t)))\ge m_{\min},\ T_k>0,
\ \phys\in\calD_{\rm phys},\\
\calG_{\rm rep}:&\quad R_{\rm ph}^{(c)},\ R_{\rm ms}^{(c)},\ R_{\rm BVP}^{(c)},\ R_{\vbar}^{(c)},\ R_{\rm phys}^{(c)},\\
\calG_{\rm mech}:&\quad R_{\rm supp}^{\rm fin},\ R_{\rm id}^{\rm fin},\ R_{\rm car}^{\rm fin},\ R_z^{\rm fin},\ R_{\rm rhs}^{\rm fin},\\
\calG_{\rm POE}:&\quad R_{\rm POE}^{\rm fin},\\
\calG_{\rm phys}:&\quad \calC_{\rm pair},\ R_g^{\rm fin}.
\end{align}
\end{subequations}

All reported observables are recomputed from accepted cycles.  Modal observables are
\begin{equation}
\calO_{\rm modal}=(T_k,A_{\IP},A_{\AP},\rho_{\IP},\rho_{\AP},E_{\perp,\min},E_{\perp,\max}),
\label{eq:modal_obs}
\end{equation}
where
\begin{equation}
\begin{aligned}
E_{\perp,\min}&=\min_t\Eperp(r(t),\sigma(t)),\\
E_{\perp,\max}&=\max_t\Eperp(r(t),\sigma(t)).
\end{aligned}
\label{eq:Eperp_range}
\end{equation}
Certificate observables are the active representation diagnostics and the unscaled mechanical final rows,
\begin{equation}
\begin{aligned}
\calO_{\rm cert}=(&\calC_k^{\rm rep},\ R_{\rm supp}^{\rm fin},R_{\rm id}^{\rm fin},R_{\rm car}^{\rm fin},R_z^{\rm fin},\\
&R_{\rm rhs}^{\rm fin},R_{\rm POE}^{\rm fin},R_g^{\rm fin}).
\end{aligned}
\label{eq:cert_obs}
\end{equation}
and locomotion observables are
\begin{equation}
\begin{aligned}
\calO_{\rm loc}&=(\vbar_k,\Delta g_k,\text{stride},\text{phase timing}),\\
\vbar_k&=\frac1{T_k}\int_0^{T_k}v_k(t)\,\dd t.
\end{aligned}
\label{eq:loc_obs}
\end{equation}

\subsection{Validation Quantities and Evidential Roles}
\label{subsec:3seg_validation_roles}

The accepted reconstructed-cycle rows and the diagnostic summaries have different roles.  Accepted rows certify natural locomotion: they report AP/IP counts, sector identity, internal return, same-physical equality, POE residuals, dynamics consistency, and nonzero displacement.  Sampled diagnostics explain the geometry of already accepted cycles: period spans, activity spans, modal shares, plotted POE values, and speed-energy views.  The frozen-vector row anchors the AP and IP families to the same physical architecture.  Control experiments test weaker alternatives, such as large internal activity, sector identity alone, AP period representations, opened motion without POE, or isolated physical ingredients.

The speed \(\vbar\) is the opened-cycle mean \(T^{-1}\int_0^T v(t)\,\dd t\); if a chart carries an active speed coordinate, that coordinate is checked against the reconstructed cycle before reporting.  The diagnostic \(\Anonv\) denotes the peak-to-peak amplitude over the non-propulsive internal coordinates, normalized in the reported numerical coordinates; it is not an acceptance row.  The POE residual \(\Psipoe\) is a final naturality row only when evaluated on an accepted opened cycle.  The right-hand-side residual \(R_{\rm rhs}^{\rm fin}\) is the reconstructed-cycle defect in \eqref{eq:rhs}; when an interior value \(\Rrhs\) is reported, it is the same defect restricted away from endpoint effects.  The symbol \(\thetaStar\) denotes the fixed passive parameter vector in \eqref{eq:3seg_phys_vector_model} whose numerical values are listed in Table~\ref{tab:3seg_theta_star}.  The same-physical pair certificate in \eqref{eq:paired_certificate} is the numerical object supporting the statement that this one frozen vector carries two modal NLM families.

\section{2SEG Results: Exact Scalar Selection}
\label{sec:2seg_results}

The 2SEG experiment isolates the scalar case of the natural-locomotion principle.  The internal oscillator has one effective transverse degree of freedom.  After the propulsive coordinate is allowed to drift, the non-POE recurrence conditions and the speed-fixing condition leave one scalar transverse condition to choose the natural cycle.  The direct construction answers yes in the reported regular regime.  The accepted opened cycles solve
\begin{equation}
  \Rex=(I_{\POE},C_u,C_Q,S)^{\mathsf T}.
\label{eq:2seg_results_residual}
\end{equation}
The condition \(I_{\POE}=0\) means that the propulsive coordinate gives no net energy to the internal oscillator over one cycle; \(C_u=0\) and \(C_Q=0\) close the remaining opened carrier variables; and \(S=0\) fixes the signed-speed slice.

\begin{table*}[!t]
\centering
\caption{Exact opened-cycle closure certificate for the scalar 2SEG Natural Locomotion Manifold family.  The 29 certified entries are locomotor cycles with nonzero signed mean speed; \(I_{\rm POE}\) is the integrated Propulsion--Oscillator Exchange.}
\label{tab:2seg_exact_certificate}
\footnotesize
\setlength{\tabcolsep}{4pt}
\renewcommand{\arraystretch}{1.10}
\begin{tabularx}{\textwidth}{@{}l c X@{}}
\toprule
Certified quantity & Maximum or count & External reading \\
\midrule
Accepted opened cycles & 29/29 & All reported cycles close the full exact certificate. \\
Signed mean-speed span & \([-24.87,25.09]\) & The certified set contains locomotion in both body directions. \\
Scaled certificate norm \(\norm{\Rex}_{\rm scaled}\) & \(5.51\times10^{-12}\) & The full four-condition certificate is closed at numerical precision. \\
Opened exchange balance \(\abs{I_{\POE}}\) & \(6.02\times10^{-13}\) & Net propulsion--oscillator exchange cancels over the cycle. \\
First internal recurrence check \(\abs{C_u}\) & \(2.30\times10^{-13}\) & One non-POE return condition closes. \\
Second internal recurrence check \(\abs{C_Q}\) & \(9.86\times10^{-12}\) & The second non-POE return condition closes. \\
Speed-fixing condition \(\abs{S}\) & \(9.53\times10^{-13}\) & The cycle lies on the selected signed-speed slice. \\
\bottomrule
\end{tabularx}
\end{table*}

The exact closure certificate defines the direct family.  Continuation is used only afterward as an independent reference.  It checks whether a separately generated opened branch recovers the same family-level observables and loop geometry.  Figure~\ref{fig:2seg_overlay} visualizes this validation by overlaying exact exchange-return cycles with the nearest continuation cycles in three phase-plane projections.  The comparison is not time-synchronized; direct and continuation cycles may use different phase origins, sampling, or arclength parametrizations.

\begin{table*}[!t]
\centering
\caption{Independent continuation comparison for the 29 exact opened-cycle closure cycles.  Loop distances compare closed curves after cyclic alignment and coordinate scaling.}
\label{tab:2seg_reference_comparison}
\footnotesize
\setlength{\tabcolsep}{4pt}
\renewcommand{\arraystretch}{1.10}
\begin{tabularx}{\textwidth}{@{}l c X@{}}
\toprule
Comparison diagnostic & Max over certified cycles & Meaning \\
\midrule
Period difference \(\abs{\Delta T}\) & \(1.50\times10^{-5}\) & The direct and reference families have nearly the same cycle period. \\
Peak shape-amplitude difference & \(1.86\times10^{-5}\) & The internal oscillation amplitude agrees with the reference branch. \\
Peak yaw-rate difference & \(8.36\times10^{-4}\) & The yaw-rate scale is recovered to sub-\(10^{-3}\) error. \\
Scaled loop RMS & \(3.07\times10^{-3}\) & RMS distance between aligned and scaled closed phase loops. \\
Scaled Chamfer distance & \(2.14\times10^{-3}\) & Symmetric point-set distance between aligned closed loops. \\
\bottomrule
\end{tabularx}
\end{table*}

\begin{figure*}[!t]
\centering
\includegraphics[width=1.00\textwidth]{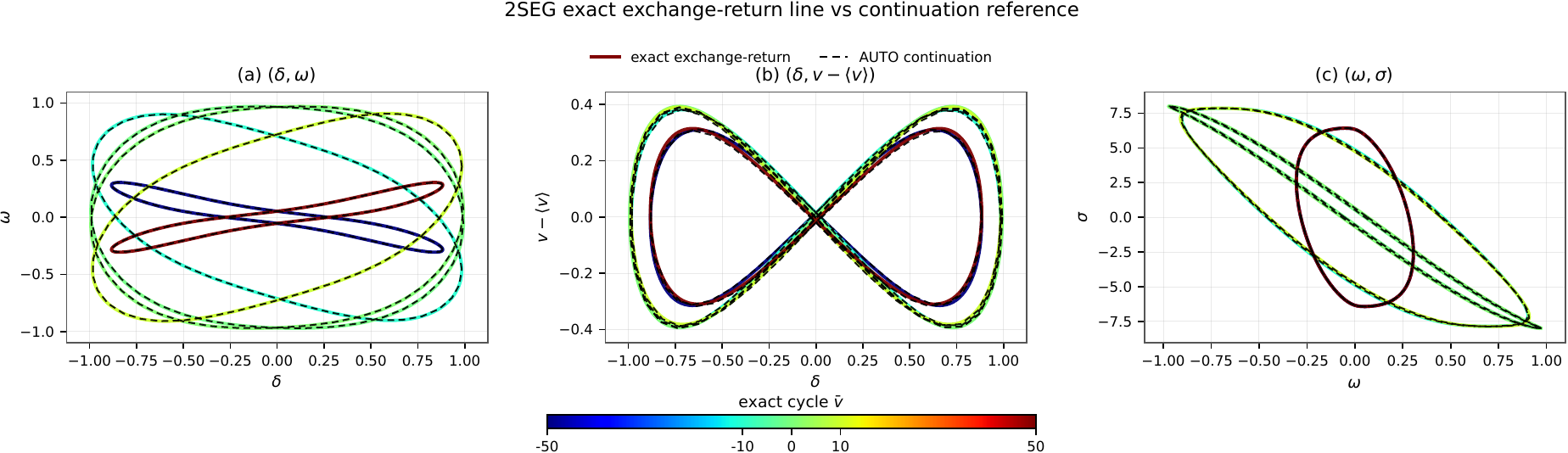}
\caption{Multi-speed geometry comparison between exact opened-cycle closure cycles and the independent continuation reference.  Solid colored curves are exact exchange-return cycles; black dashed curves are nearest continuation cycles.  The panels show internal angle/yaw rate \((\delta,\omega)\), internal angle/centered speed \((\delta,v-\langle v\rangle)\), and yaw rate/internal rate \((\omega,\sigma)\).}
\label{fig:2seg_overlay}
\end{figure*}

The physical content is that the internal oscillator returns, the body coordinate drifts, and the net energy exchange between propulsion and the internal oscillator cancels over the locomotor period.  In the scalar case, the exact theorem explains why this exchange-return row closes the missing coordinate.  In the 3SEG case, the same condition must be embedded in a modal certificate.

\section{3SEG Results: Same-Physical Modal Natural Locomotion}
\label{sec:3seg_results}

The 3SEG experiment asks the design question that the scalar theorem cannot answer: can one passive mechanical architecture contain more than one natural-locomotion family?  The result is affirmative: the fixed passive vector \(\thetaStar\), with parameter meanings given in Table~\ref{tab:3seg_parameter_definitions} and numerical values in Table~\ref{tab:3seg_theta_star}, supports both an in-phase family and an anti-phase family,
\begin{equation}
  \theta^{\IP}_{\rm phys}=\theta^{\AP}_{\rm phys}=\thetaStar,
  \qquad
  \calN_{\IP}(\thetaStar)\ne\varnothing,
  \qquad
  \calN_{\AP}(\thetaStar)\ne\varnothing .
\label{eq:3seg_same_physical_claim}
\end{equation}
The claim is not that two motions can be retuned separately.  It is that two opened-channel NLM families are recovered on the same passive architecture, with no family-specific change of masses, inertias, geometry, or stiffnesses.  The surface visualization in Fig.~\ref{fig:3seg_surface} now becomes the result-level picture of this claim: its top six plots are anti-phase samples and its bottom six plots are in-phase samples, all drawn from accepted cycles at the fixed architecture and all interpreted through the certificate rows below.

\begin{table}[!t]
\caption{Family-level certificate summary at the same physical vector \(\thetaStar\). AP denotes anti-phase, IP denotes in-phase, POE is Propulsion--Oscillator Exchange, and the residual maxima are computed over the full accepted traces.}
\label{tab:3seg_family_certificate}
\centering
\scriptsize
\setlength{\tabcolsep}{1.4pt}
\renewcommand{\arraystretch}{1.08}
\begin{tabularx}{\columnwidth}{@{}p{0.39\columnwidth}YY@{}}
\toprule
Quantity & AP family & IP family \\
\midrule
Accepted opened cycles & 2278 & 1358 \\
Sector-identity pass & 2278/2278 AP & 1358/1358 IP \\
Internal return & accepted periodic opened trace & accepted periodic opened trace \\
Physical parameters & \(\thetaStar\) & \(\thetaStar\) \\
Mean-speed span \(\vbar\) & [-126.96, 122.79] & [-155.56, 166.11] \\
Speed-coordinate source & active branch speed & recomputed cycle mean \\
Max \(\Rpoe=\abs{\Psipoe}\) & \(1.69\!\times\!10^{-8}\) & \(2.76\!\times\!10^{-8}\) \\
Max \(\Rrhs\) & \(1.10\!\times\!10^{-3}\) & \(1.50\!\times\!10^{-3}\) \\
\bottomrule
\end{tabularx}
\end{table}

The surface figure turns the certificate into a reader-facing picture: AP and IP occupy distinct sampled organizations of speed, angular, and transverse-energy variables while sharing the same passive vector.  The sampled ranges in Table~\ref{tab:3seg_sampled_diagnostics} then quantify this contrast without redefining acceptance.

\begin{table}[!t]
\caption{Diagnostic ranges from 72 sampled accepted cycles per family.  The quantity \(\Anonv\) is the normalized peak-to-peak amplitude over non-propulsive internal coordinates; these diagnostics interpret modal contrast but do not replace the branch-level counts in Table~\ref{tab:3seg_family_certificate}.}
\label{tab:3seg_sampled_diagnostics}
\centering
\scriptsize
\setlength{\tabcolsep}{1.5pt}
\renewcommand{\arraystretch}{1.08}
\begin{tabularx}{\columnwidth}{@{}p{0.38\columnwidth}YY@{}}
\toprule
Quantity & AP sample & IP sample \\
\midrule
Period span & [0.5563, 0.6295] & [0.8452, 0.8642] \\
\(\Anonv\) & [52.23, 84.43] & [4.915, 5.388] \\
Expected-sector share & AP 0.704--0.818 & IP 0.961--0.965 \\
Modal ratio & AP/IP 2.38--4.48 & IP/AP 24.53--27.47 \\
Max sampled \(\abs{\Psipoe}\) & \(1.60\!\times\!10^{-8}\) & \(2.50\!\times\!10^{-8}\) \\
Max sampled \(\Rrhs\) & \(9.15\!\times\!10^{-4}\) & \(1.49\!\times\!10^{-4}\) \\
\bottomrule
\end{tabularx}
\end{table}

\begin{table*}[!t]
\caption{Frozen physical architecture \(\thetaStar\) used for the same fixed-parameter AP/IP result.  AP and IP denote anti-phase and in-phase modal families.  Ratios are reported in the normalized numerical parameterization and should be read as design-coordinate comparisons, not dimensional laws.}
\label{tab:3seg_theta_star}
\centering
\scriptsize
\renewcommand{\arraystretch}{1.10}
\setlength{\tabcolsep}{3.0pt}
\begin{tabularx}{\textwidth}{@{}p{0.16\textwidth}p{0.43\textwidth}X@{}}
\toprule
Block & Values in \(\thetaStar\) & Mechanical reading in normalized coordinates \\
\midrule
Inertias & \(I_1=0.418642\), \(I_2=0.00608070\), \(I_3=0.000314296\). & Strong proximal inertial scale: \(I_1/I_2=68.85\), \(I_2/I_3=19.35\), \(I_1/I_3=1332\). \\
Geometry and masses & \(l_1=0.28\), \(L_2=0.150076\), \(L_3=0.105188\), \(m_1=6.0\), \(m_2=0.394571\), \(m_3=0.276554\). & Heavy proximal body and compact downstream links: \(m_1/m_2=15.21\), \(m_2/m_3=1.43\), \(l_1/L_3=2.66\). \\
Elastic and coupling terms & \(k_{12}=0.00118731\), \(k_{2,1}=0.652435\), \(k_{2,2}=0.300805\), \(k_{4,1}=5.258424\), \(k_{4,2}=0.449143\). & Asymmetric local stiffness with small normalized inter-joint coupling: \(k_{4,1}/k_{4,2}=11.71\), \(k_{2,1}/k_{2,2}=2.17\). \\
Actuation inputs & \(u_1=0\), \(u_2=0\). & Strictly conservative no-actuation setting; these symbols are torque inputs, not passive offsets. \\
\bottomrule
\end{tabularx}
\end{table*}

Same-physicality is the design content of the result: it turns two reconstructed modal families into evidence about one passive mechanical architecture.  The common mechanism combines proximal inertial leverage, compact downstream mass and length scales, asymmetric local stiffness, and small normalized inter-joint coupling.

\begin{figure*}[!t]
\centering
\makebox[\textwidth][c]{%
\includegraphics[width=0.54\textwidth]{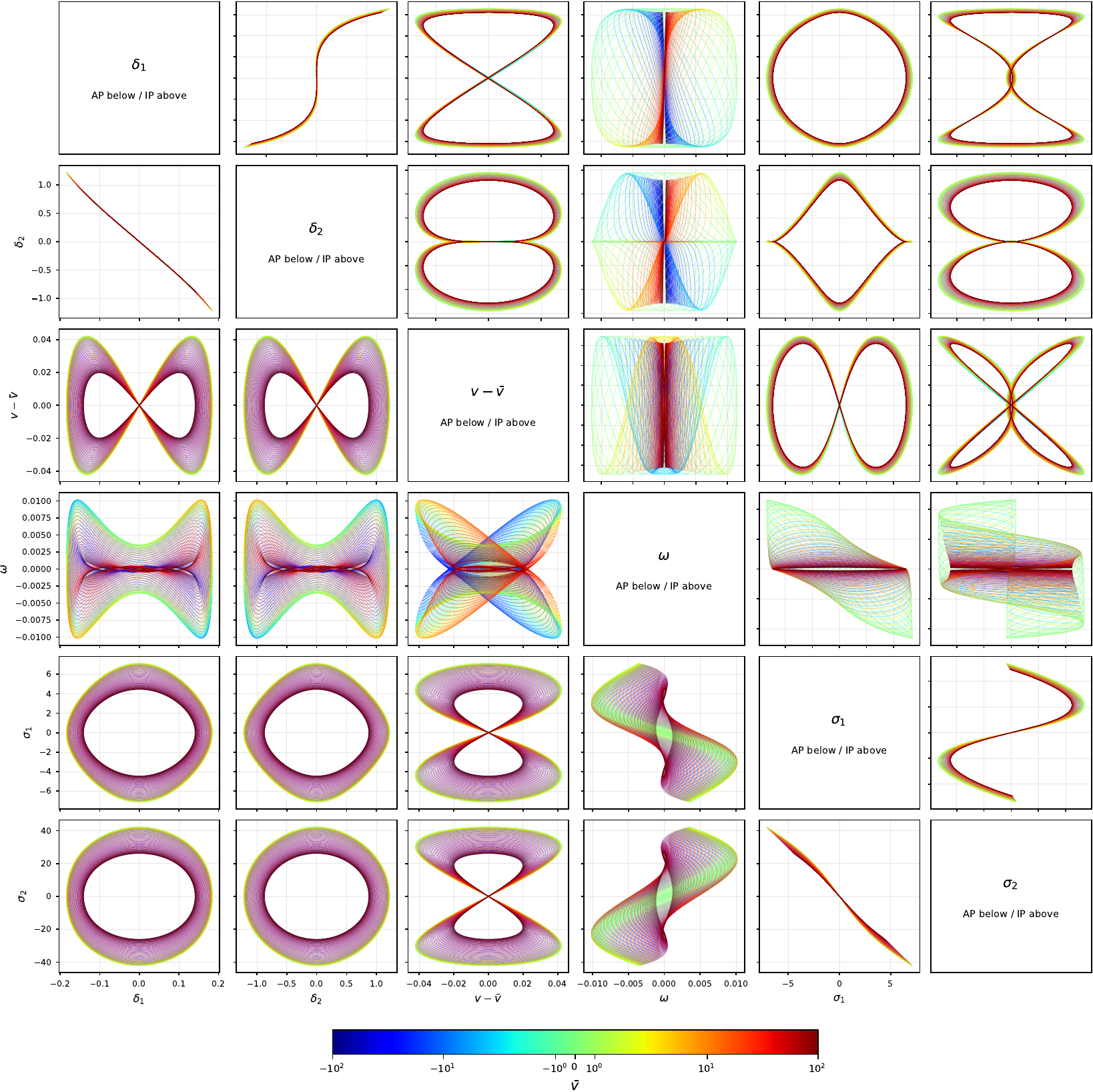}%
\hspace{0.01\textwidth}%
\includegraphics[width=0.45\textwidth]{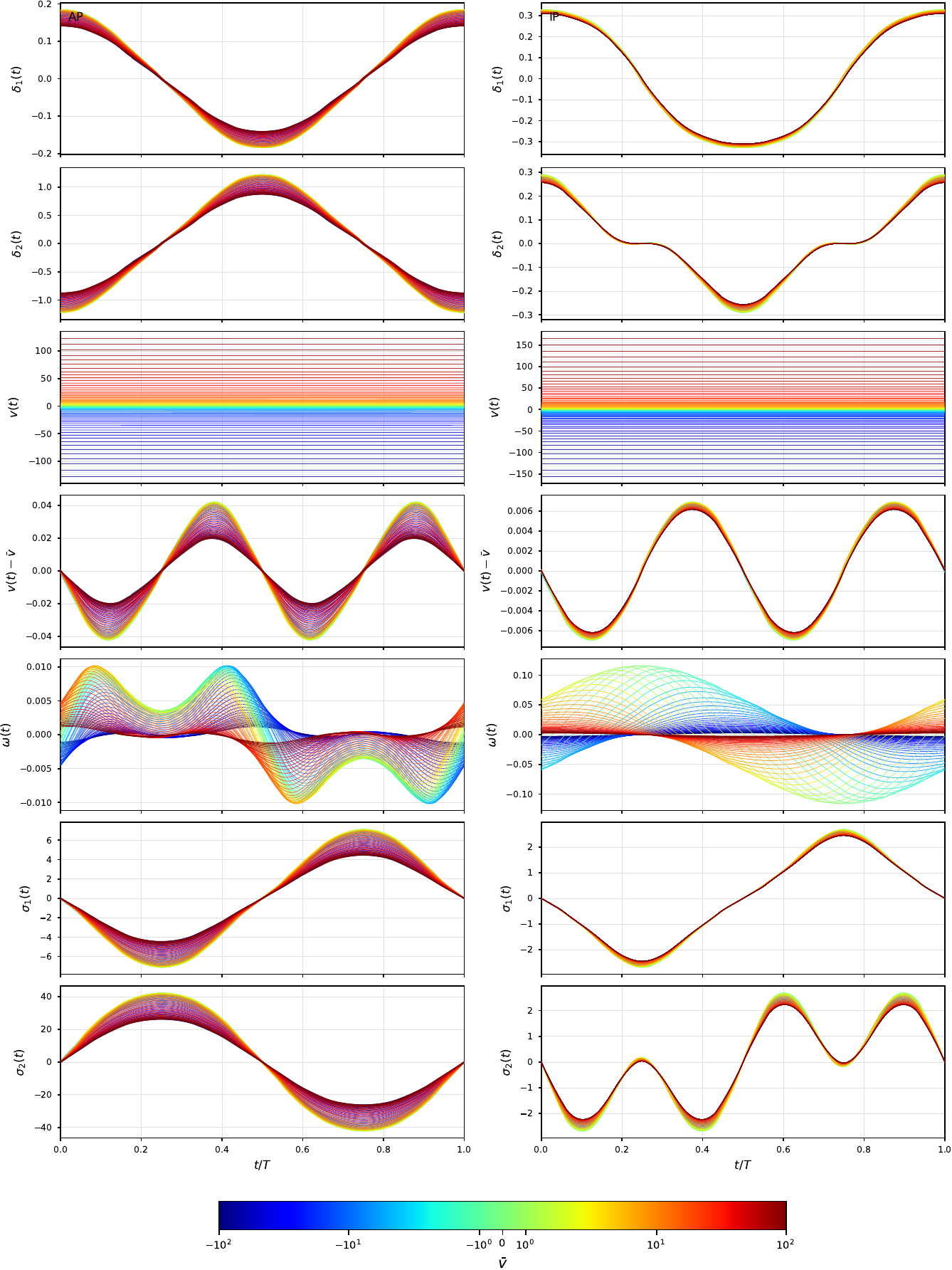}%
}
\caption{Accepted same fixed-parameter anti-phase/in-phase families.  Left: enlarged phase-pair geometry, with anti-phase cycles in the lower triangle and in-phase cycles in the upper triangle.  Right: enlarged time-domain choreography over one displayed period.  Together they show modal separation in projection and one-period choreography for the accepted same-physical families.}
\label{fig:3seg_pair_time}
\end{figure*}

Figure~\ref{fig:3seg_pair_time} gives the geometric and temporal reading of the positive result.  It shows that the accepted families have different internal choreographies; the certificate remains the opened-cycle acceptance in Tables~\ref{tab:3seg_family_certificate}--\ref{tab:3seg_theta_star}.

\subsection{3SEG Control Experiments}
\label{subsec:3seg_controls}

The first control sequence tests whether weaker naturality cues can replace opened-cycle NLM acceptance.  It progresses from internal amplitude, to sector identity, to AP representation choices, to opened motion without POE, and finally to the full opened POE-certified route.

\begin{table*}[!t]
\caption{Naturality controls.  Each row tests whether a weaker cue can replace the opened-cycle Natural Locomotion Manifold (NLM) acceptance test.}
\label{tab:3seg_naturality_controls}
\centering
\scriptsize
\setlength{\tabcolsep}{2.2pt}
\renewcommand{\arraystretch}{1.11}
\begin{tabularx}{\textwidth}{@{}p{0.155\textwidth}p{0.275\textwidth}p{0.255\textwidth}X@{}}
\toprule
Alternative tested & Control & What survives & What fails / lesson \\
\midrule
Large internal activity alone & Choose AP-looking closed supports with large non-propulsive amplitude, then reopen propulsion. & Large amplitudes: targeted controls reached \(\Anonv=225.66\)--369.96; a larger AP-oriented finite screen reached 60.34--757.27. & No moving AP family was recovered.  The largest AP-like amplitudes appear in rejected controls, so amplitude is not the selection principle. \\
\addlinespace[0.25em]
Sector identity alone & Test candidates that retain the IP or AP sector without full opened-cycle acceptance. & Several controls preserve sector identity, including stationary IP-like and AP-like cycles. & Sector identity organizes candidates, but it does not produce locomotion by itself. \\
\addlinespace[0.25em]
AP period-only descriptions & Represent an AP internal support as a single-period or three-period object while not keeping opened mean-speed/exchange active. & The AP sector description remains available. & Both period-only choices stayed stationary: 0/4 moving cases for single-period and 0/4 for three-period descriptions. \\
\addlinespace[0.25em]
AP opened speed/exchange form & Use an opened AP formulation where the mean-speed coordinate and POE exchange remain active during reconstruction. & Movement appears in 2/4 single-period and 4/4 three-period cases. & AP movement appears only when opened propulsion speed and exchange balance remain active during reconstruction. \\
\addlinespace[0.25em]
Opened motion without POE closure & Remove the POE row from opened-cycle refinement and screen the same finite candidate set. & One identity-preserving IP-moving case survives, with 10 moving continuations; the reconstructed-cycle check reports \(\Rrhs=1.04\times10^{-1}\). & Removing POE does not remove all opened motion; it removes AP and the paired same-physical AP/IP result. \\
\addlinespace[0.25em]
Full opened POE-certified route & Accept only cycles that pass sector identity, internal return, POE closure, dynamics consistency, and nonzero displacement. & AP 2278/2278 and IP 1358/1358 accepted cycles are recovered at the same \(\thetaStar\). & The recovered object is a pair of modal NLM families, not merely oscillation amplitude, a sector label, an AP representation, or motion alone. \\
\bottomrule
\end{tabularx}
\end{table*}

Table~\ref{tab:3seg_naturality_controls} separates modal organization from natural locomotion.  Large AP-like oscillations can be manufactured without recovering an AP family, and sector-compatible candidates can remain stationary.  The no-POE row is discriminating: within the tested construction, opened motion can survive without POE, but only as a partial IP-side result with a large dynamics mismatch; the AP family and paired AP/IP certificate disappear.  POE is therefore the opened-channel exchange condition that filters motion into natural locomotion, while modal identity and same-physical certification supply the information missing from a scalar condition.

The second control sequence asks whether same-physical AP/IP multiplicity comes from one isolated physical lever.  If one ingredient were sufficient, then an inertia-only, stiffness/coupling-only, or geometry/mass-only control should recover both moving families.  Instead, each isolated design preserves only a fragment of the final certificate.

\begin{table*}[!t]
\caption{Mechanical-design controls.  Isolated controls preserve fragments of the certificate, but none recovers the same fixed-parameter anti-phase/in-phase pair.}
\label{tab:3seg_mechanical_controls}
\centering
\scriptsize
\setlength{\tabcolsep}{2.2pt}
\renewcommand{\arraystretch}{1.11}
\begin{tabularx}{\textwidth}{@{}p{0.18\textwidth}p{0.25\textwidth}p{0.25\textwidth}X@{}}
\toprule
Isolated ingredient & What survives & What fails & Design lesson \\
\midrule
Elastic/coupling set & AP and IP sector identity can be preserved at small amplitude: AP \(\Anonv\approx3.87\)--3.90, IP \(\Anonv\approx0.18\). & All moving counts are zero. & Local stiffness and coupling shape modal oscillators, but in isolation they do not complete a locomotor mechanism. \\
\addlinespace[0.25em]
Primary inertial scale & IP and AP identity pass.  AP-like activity rises to \(\Anonv\approx8.88\)--9.07, while IP remains near 0.25. & No case moves. & The inertial lever preferentially amplifies AP-like activity, but amplification alone does not produce AP locomotion or AP/IP multiplicity. \\
\addlinespace[0.25em]
Downstream geometry/mass distribution & The IP control preserves identity with \(\Anonv=6.76\).  AP variants produce very large activity, \(\Anonv=556.72\)--570.31. & IP does not move.  AP variants lose AP identity and have large dynamics defects. & Geometry/mass placement creates the strongest AP-like false positives when the rest of the architecture is incomplete. \\
\addlinespace[0.25em]
Partial physical designs & In 24 finite designs assembled from subsets of the tested ingredients, four IP variants pass identity. & No variant moves, and no AP variant passes identity. & IP identity is comparatively easier to preserve under partial designs, but identity-stable IP oscillations are not automatically locomotor. \\
\addlinespace[0.25em]
Complete coupled architecture & The full \(\thetaStar\) architecture gives AP 2278/2278 and IP 1358/1358 accepted cycles. & None of the isolated controls reproduces the paired result. & Modal multiplicity is recovered as a property of the coupled passive architecture, not as a single-knob effect. \\
\bottomrule
\end{tabularx}
\end{table*}

The mechanical failures are asymmetric.  Elastic/coupling-only controls preserve local modal organization but stay stationary.  The primary inertial control amplifies AP-like activity more than IP-like activity, yet it does not turn either sector into a moving opened family.  The downstream geometry/mass control is the strongest false-positive mechanism: it creates very large AP-like motion while losing AP identity and dynamics consistency.  Conversely, IP identity survives several partial mechanisms, but those identity-stable IP candidates remain non-locomotor.  These controls support a bounded design interpretation: AP-like activity is easy to amplify but hard to certify, whereas IP identity is easier to preserve but still requires the complete architecture to become an NLM.

\FloatBarrier
\section{Discussion}
\label{sec:discussion}

The two realizations show different mathematical statuses of the same principle.  In 2SEG, the scalar closed support and the positive oriented section make storage equality equivalent to coordinate equality.  Once the carrier-variable return and mean-speed normalization conditions close, zero POE closes the remaining transverse coordinate.  That is why the 2SEG result is an exact exchange-return theorem and why the numerical implementation can be read as a direct scalar computation rather than as a continuation discovery.

The comparison below summarizes the resulting status difference.  The distinction is not that the 2SEG is theoretical and the 3SEG is numerical.  Both use mechanics and both are computed.  The distinction is the dimension of the remaining transverse information after carrier and normalization rows have closed.  In the 2SEG, one scalar storage row identifies the returned section point.  In the 3SEG, one scalar storage row identifies only a hypersurface inside the transverse section, so modal identity and branch membership remain active parts of the claim.

\begin{table*}[!t]
\caption{Complementary mathematical status of the two realizations.}
\label{tab:2seg_3seg_status}
\centering
\footnotesize
\setlength{\tabcolsep}{3pt}
\renewcommand{\arraystretch}{1.12}
\begin{tabularx}{\textwidth}{@{}p{0.16\textwidth}p{0.25\textwidth}p{0.25\textwidth}X@{}}
\toprule
Aspect & 2SEG scalar realization & 3SEG modal realization & Consequence\\
\midrule
Effective internal dimension & one transverse coordinate after orientation & two transverse modal directions & storage equality is complete only in 2SEG\\
Closed object & scalar carrier-closed support & IP/AP carrier-closed modal supports & 3SEG must preserve sector identity\\
Opened variable & finite-speed yaw-propulsion return & lifted modal cycle with carrier and internal return & 3SEG needs a larger reconstructed-cycle certificate\\
POE status & missing scalar section row; theorem-grade equivalence under the regular annulus & necessary exchange row embedded in a modal certificate & scalar sufficiency does not generalize directly\\
Numerical role of continuation & independent geometric reference after direct solve & candidate generation and transport among charts & branch membership alone is never a certificate\\
Design interpretation & exact line of finite-speed scalar NLM cycles & one frozen passive architecture supporting two modal NLM families & 2SEG validates the theorem; 3SEG demonstrates modal multiplicity\\
\bottomrule
\end{tabularx}
\end{table*}

This separation also clarifies what should be compared across systems.  The 2SEG should not be judged by the number of modal families it supports; its contribution is exactness.  The 3SEG should not be judged by whether POE alone closes a state; its contribution is the same-physical modal certificate.  The shared scientific content is the closed/open exchange principle.  The different mathematical statuses are consequences of internal dimension, not changes of definition.

In 3SEG, the scalar obstruction is structural.  The transverse state has multiple coordinates, and equality of \(\Eperp\) does not identify a point on the energy surface.  Modal organization must first select a closed-channel family.  The opened lift must then preserve sector identity, close the internal state, close POE, pass the right-hand-side dynamics check, keep the physical vector fixed, and reconstruct nonzero displacement.  The 3SEG result is therefore not a second scalar theorem; it is a same-physical modal certificate.

The scope is deliberately finite.  The results do not classify all 2SEG or 3SEG architectures, prove optimality of \(\thetaStar\), or assert scalar POE sufficiency for general multi-effective-coordinate locomotors.  The essential object is not nonholonomicity itself, but environment-mediated exchange.  In this paper the mediator is an ideal continuous nonholonomic constraint, so the exchange is conservative and POE closure is an internal balance.  In walking, the mediator would be discontinuous contact; in swimming or flying, it would be a continuous but non-ideal fluid interaction with losses and possibly inputs.  The mathematical balance must change, but the organizing question remains the same.  The conservative ideal-constraint setting therefore excludes dissipative fluids, frictional losses, impacts, and forced actuation.  Its value is as a baseline: once dissipation is introduced, one may expect conservative NLM families to deform into resonance-locomotion families in which forcing compensates losses without imposing arbitrary trajectories.  That extension is outside the present paper.

The methodological distinction is nevertheless broad.  A geometric gait asks what displacement a shape loop reconstructs.  A continuation study asks what family of solutions can be tracked.  The NLM construction asks whether a closed-channel natural oscillator survives the opened propulsive channel with zero net exchange, internal recurrence, and pose drift.  In one effective degree of freedom this becomes an exact scalar closure.  In several effective degrees of freedom it becomes a modal certificate and a design question: does a passive architecture support no natural family, one family, or several families with different phase organizations?

The same caution applies to optimality.  A POE-certified NLM is not claimed to minimize transport cost, maximize speed, or optimize robustness.  It is a mechanically selected locomotor family.  Optimality questions can be asked after the family is identified, but they are not the defining criterion.  This separates two design operations: first find the motions the passive architecture can sustain naturally, then decide which member is preferable for a task or under actuation.

\subsection{Reading the exchange row}
\label{subsec:exchange_row_reading}

The POE row should not be interpreted as an energetic cost.  It is a cyclic compatibility condition between the internal oscillator and the opened propulsive channel.  A positive instantaneous \(\Ppoe\) during part of the cycle and a negative instantaneous \(\Ppoe\) during another part are compatible with natural locomotion; what is forbidden is net drift of the transverse reservoir after one internal period.  This is why \(\Psipoe=0\) is different from minimizing effort or suppressing all exchange.  Propulsion can interact with the oscillator, but it must return what it borrows over the completed cycle.

In the scalar theorem, this balance has a second meaning: it is also the missing coordinate closure.  On the positive oriented section, storage is a strictly monotone function of the scalar crossing speed.  Therefore zero POE is not merely a physical sentence; it is an algebraic return condition.  This is the strongest possible status of the row and is the reason the 2SEG occupies a special place in the paper.

In the modal certificate, the same row loses that sufficiency but retains its physical meaning.  It says that the opened channel is exchange-balanced on the reconstructed IP or AP cycle.  It does not say that the cycle has the correct modal identity, that the internal state returns, or that the same physical vector is used for both families.  Those claims require their own rows.  The 3SEG construction is therefore not weaker because POE is less important; it is more explicit about the additional information that a multi-effective-degree-of-freedom oscillator needs.

\subsection{Scope of the conservative claim}
\label{subsec:conservative_page}

Several natural extensions are intentionally left out of the main claim.  The paper does not optimize the passive architecture, globally classify 3SEG modal sectors, or claim that a scalar exchange row is sufficient in higher dimension.  It also does not replace forced or dissipative locomotion theory.  Its role is to provide a conservative reference family against which such extensions can later be organized.  Hybrid walkers may add piecewise-holonomic or discontinuous contacts; swimmers and flyers may add continuous non-ideal media and actuation, as in forced-oscillation robotic horizons such as \cite{Rajaomarosata2025OCEANS}; in each case the conservative NLM would serve as a baseline, not as a finished theory.

These exclusions are not weaknesses of the formulation.  They keep the contribution testable.  The 2SEG claim can be challenged by checking the regular annulus assumptions and the exchange-return rows.  The 3SEG claim can be challenged by recomputing the same-physical paired certificate, the sector-identity check, the right-hand-side residuals, and the controls.  A broader theory may later add dissipation, forcing, robustness, or optimality, but those extensions should preserve the central distinction made here: a natural locomotion family is first selected by mechanical exchange and recurrence, and only then optimized, actuated, or stabilized.

\section{Conclusion}
\label{sec:conclusion}

Natural locomotion can be characterized as a mechanically selected family, not merely as a closed shape loop or a continued trajectory.  The proposed selection rule is internal recurrence with nonzero pose drift and zero cyclic Propulsion--Oscillator Exchange.  The method closes the propulsive channel to expose the internal oscillator, then opens it and accepts only cycles whose exchange, internal state, dynamics, fixed-parameter, and displacement conditions close at the level required by their internal dimension.

The 2SEG realization gives the strongest scalar status: under a regular oriented-section assumption, the POE row is equivalent to the missing scalar transverse return, and the reported finite-speed branch matches an independent continuation reference.  The 3SEG realization gives the modal extension: one frozen passive architecture supports both in-phase and anti-phase POE-certified NLM families, while controls show why amplitude, sector labels, opened motion, or isolated physical ingredients do not replace the full certificate.  The resulting design question is explicit: which passive architectures support no natural locomotion family, which support one, and which support several mechanically distinct families?  The article therefore identifies natural locomotion with an exchange-balanced invariant family: the environment-mediated propulsive channel may exchange power with the internal oscillator, but a natural cycle returns that exchange to zero over one period.

\bibliographystyle{IEEEtran}
\bibliography{references}

\begin{thebibliography}{10}
\providecommand{\url}[1]{#1}
\csname url@samestyle\endcsname
\providecommand{\newblock}{\relax}
\providecommand{\bibinfo}[2]{#2}
\providecommand{\BIBentrySTDinterwordspacing}{\spaceskip=0pt\relax}
\providecommand{\BIBentryALTinterwordstretchfactor}{4}
\providecommand{\BIBentryALTinterwordspacing}{\spaceskip=\fontdimen2\font plus
\BIBentryALTinterwordstretchfactor\fontdimen3\font minus
  \fontdimen4\font\relax}
\providecommand{\BIBforeignlanguage}[2]{{%
\expandafter\ifx\csname l@#1\endcsname\relax
\typeout{** WARNING: IEEEtran.bst: No hyphenation pattern has been}%
\typeout{** loaded for the language `#1'. Using the pattern for}%
\typeout{** the default language instead.}%
\else
\language=\csname l@#1\endcsname
\fi
#2}}
\providecommand{\BIBdecl}{\relax}
\BIBdecl

\bibitem{McGeer1990}
T.~McGeer, ``Passive dynamic walking,'' \emph{The International Journal of
  Robotics Research}, vol.~9, no.~2, pp. 62--82, 1990.

\bibitem{Collins2005}
S.~Collins, A.~Ruina, R.~Tedrake, and M.~Wisse, ``Efficient bipedal robots
  based on passive-dynamic walkers,'' \emph{Science}, vol. 307, no. 5712, pp.
  1082--1085, 2005.

\bibitem{Kashiri2018}
N.~Kashiri, A.~Abate, S.~J. Abram, A.~Albu-Sch{\"a}ffer, P.~J. Clary,
  S.~Faraji, R.~Furn{\'e}mont, M.~Garabini, A.~Margan, F.~Negrello,
  N.~Tsagarakis, D.~G. Caldwell, A.~Bicchi, and C.~Semini, ``An overview on
  principles for energy efficient robot locomotion,'' \emph{Frontiers in
  Robotics and AI}, vol.~5, p. 129, 2018.

\bibitem{Schonebaum2021}
J.~K. Schonebaum, F.~Alijani, and G.~Radaelli, ``Review on mobile robots that
  exploit resonance,'' \emph{Proceedings of the Institution of Mechanical
  Engineers, Part C: Journal of Mechanical Engineering Science}, vol. 235,
  no.~24, pp. 7907--7924, 2021.

\bibitem{Blickhan1989}
R.~Blickhan, ``The spring-mass model for running and hopping,'' \emph{Journal
  of Biomechanics}, vol.~22, no. 11--12, pp. 1217--1227, 1989.

\bibitem{FullKoditschek1999}
R.~J. Full and D.~E. Koditschek, ``Templates and anchors: Neuromechanical
  hypotheses of legged locomotion on land,'' \emph{Journal of Experimental
  Biology}, vol. 202, no.~23, pp. 3325--3332, 1999.

\bibitem{Holmes2006}
P.~Holmes, R.~J. Full, D.~Koditschek, and J.~Guckenheimer, ``The dynamics of
  legged locomotion: Models, analyses, and challenges,'' \emph{SIAM Review},
  vol.~48, no.~2, pp. 207--304, 2006.

\bibitem{Rosenberg1962}
R.~M. Rosenberg, ``The normal modes of nonlinear n-degree-of-freedom systems,''
  \emph{Journal of Applied Mechanics}, vol.~29, no.~1, pp. 7--14, 1962.

\bibitem{ShawPierre1993}
S.~W. Shaw and C.~Pierre, ``Normal modes for non-linear vibratory systems,''
  \emph{Journal of Sound and Vibration}, vol. 164, no.~1, pp. 85--124, 1993.

\bibitem{Kerschen2009PartI}
G.~Kerschen, M.~Peeters, J.-C. Golinval, and A.~F. Vakakis, ``Nonlinear normal
  modes, Part {I}: A useful framework for the structural dynamicist,''
  \emph{Mechanical Systems and Signal Processing}, vol.~23, no.~1, pp.
  170--194, 2009.

\bibitem{Peeters2009PartII}
M.~Peeters, R.~Vigui{\'e}, M.~S{\'e}n{\'e}chal, G.~Kerschen, and J.-C.
  Golinval, ``Nonlinear normal modes, Part {II}: Toward a practical computation
  using numerical continuation techniques,'' \emph{Mechanical Systems and
  Signal Processing}, vol.~23, no.~1, pp. 195--216, 2009.

\bibitem{AlbuSchaeffer2020}
A.~Albu-Sch{\"a}ffer and C.~Della~Santina, ``A review on nonlinear modes in
  conservative mechanical systems,'' \emph{Annual Reviews in Control}, vol.~50,
  pp. 49--71, 2020.

\bibitem{HallerPonsioen2016}
G.~Haller and S.~Ponsioen, ``Nonlinear normal modes and spectral submanifolds:
  Existence, uniqueness and use in model reduction,'' \emph{Nonlinear
  Dynamics}, vol.~86, pp. 1493--1534, 2016.

\bibitem{DellaSantina2021ISER}
C.~Della~Santina, D.~Lakatos, A.~Bicchi, and A.~Albu-Schaeffer, ``Using
  nonlinear normal modes for execution of efficient cyclic motions in
  articulated soft robots,'' in \emph{Experimental Robotics}.\hskip 1em plus
  0.5em minus 0.4em\relax Cham: Springer, 2021, pp. 566--575.

\bibitem{DellaSantina2021CDC}
------, ``Exciting nonlinear modes of conservative mechanical systems by
  operating a master variable decoupling,'' in \emph{Proceedings of the 60th
  IEEE Conference on Decision and Control (CDC)}.\hskip 1em plus 0.5em minus
  0.4em\relax IEEE, 2021, pp. 2448--2455.

\bibitem{Bjelonic2022}
F.~Bjelonic, A.~Sachtler, A.~Albu-Sch{\"a}ffer, and C.~Della~Santina,
  ``Experimental closed-loop excitation of nonlinear normal modes on an elastic
  industrial robot,'' \emph{IEEE Robotics and Automation Letters}, vol.~7,
  no.~2, pp. 1689--1696, 2022.

\bibitem{Calzolari2023}
D.~Calzolari, C.~Della~Santina, A.~M. Giordano, A.~Schmidt, and
  A.~Albu-Sch{\"a}ffer, ``Embodying quasi-passive modal trotting and pronking
  in a sagittal elastic quadruped,'' \emph{IEEE Robotics and Automation
  Letters}, vol.~8, no.~4, pp. 2285--2292, 2023.

\bibitem{Calzolari2026NMM}
D.~Calzolari, C.~Della~Santina, and A.~Albu-Sch{\"a}ffer, ``Exciting families
  of passive gaits in an elastic quadruped via natural motion manifold
  control,'' \emph{The International Journal of Robotics Research}, vol.~45,
  no.~2, pp. 233--258, 2026.

\bibitem{KellyMurray1995}
S.~D. Kelly and R.~M. Murray, ``Geometric phases and robotic locomotion,''
  \emph{Journal of Robotic Systems}, vol.~12, no.~6, pp. 417--431, 1995.

\bibitem{OstrowskiBurdick1998}
J.~Ostrowski and J.~Burdick, ``The geometric mechanics of undulatory robotic
  locomotion,'' \emph{The International Journal of Robotics Research}, vol.~17,
  no.~7, pp. 683--701, 1998.

\bibitem{BlochKrishnaprasadMarsdenMurray1996}
A.~M. Bloch, P.~S. Krishnaprasad, J.~E. Marsden, and R.~M. Murray,
  ``Nonholonomic mechanical systems with symmetry,'' \emph{Archive for Rational
  Mechanics and Analysis}, vol. 136, no.~1, pp. 21--99, 1996.

\bibitem{Bloch2015}
A.~M. Bloch, \emph{Nonholonomic Mechanics and Control}, 2nd~ed.\hskip 1em plus
  0.5em minus 0.4em\relax New York: Springer, 2015.

\bibitem{Raff2022}
M.~Raff, N.~Rosa~Jr., and C.~D. Remy, ``Connecting gaits in energetically
  conservative legged systems,'' \emph{IEEE Robotics and Automation Letters},
  vol.~7, no.~3, pp. 8407--8414, 2022.

\bibitem{KuznetsovContinuation}
Y.~A. Kuznetsov, \emph{Elements of Applied Bifurcation Theory}, 3rd~ed.\hskip
  1em plus 0.5em minus 0.4em\relax New York: Springer, 2004.

\bibitem{DoedelOldeman2012}
E.~J. Doedel, A.~R. Champneys, F.~Dercole, T.~F. Fairgrieve,
  Y.~A. Kuznetsov, B.~E. Oldeman, R.~C. Paffenroth, B.~Sandstede,
  X.~Wang, and C.~Zhang, \emph{AUTO-07P: Continuation and Bifurcation
  Software for Ordinary Differential Equations}.\hskip 1em plus 0.5em minus
  0.4em\relax Montreal, Canada: Department of Computer Science, Concordia
  University, 2007.


\bibitem{Rajaomarosata2025IFAC}
M.~Rajaomarosata, L.~Jaulin, L.~Lapierre, and S.~Rohou, ``Natural efficient
  gaits from nonholonomic locomotion nonlinear normal mode ({NL-NNM}): The
  {Pendrivencar} case,'' \emph{Mechatronics}, vol. 110, p. 103366, 2025.

\bibitem{Arnold1989}
V.~I. Arnol'd, \emph{Mathematical Methods of Classical Mechanics},
  2nd~ed.\hskip 1em plus 0.5em minus 0.4em\relax New York: Springer, 1989.

\bibitem{NayfehMook1995}
A.~H. Nayfeh and D.~T. Mook, \emph{Nonlinear Oscillations}.\hskip 1em plus
  0.5em minus 0.4em\relax New York: Wiley, 1995.

\bibitem{SandersVerhulstMurdock2007}
J.~A. Sanders, F.~Verhulst, and J.~Murdock, \emph{Averaging Methods in
  Nonlinear Dynamical Systems}, 2nd~ed.\hskip 1em plus 0.5em minus 0.4em\relax
  New York: Springer, 2007.

\bibitem{Rajaomarosata2025OCEANS}
M.~Rajaomarosata, L.~Jaulin, L.~Lapierre, and S.~Rohou, ``Energy-efficient
  nonholonomic fish robot: Nonlinear forced oscillations,'' in \emph{OCEANS
  2025 -- Brest}.\hskip 1em plus 0.5em minus 0.4em\relax IEEE, 2025, pp. 1--7.

\end{thebibliography}

\end{document}